\documentclass[11pt,a4paper]{article}

\usepackage[T1]{fontenc}
\usepackage{lmodern}
\usepackage[margin=1in]{geometry}
\usepackage{amsmath,amssymb,amsthm}
\usepackage{algorithm}
\usepackage{algorithmic}
\usepackage{booktabs}
\usepackage{longtable}
\usepackage{url}
\usepackage{enumitem}
\usepackage{graphicx}
\usepackage[round,authoryear]{natbib}
\usepackage[hidelinks]{hyperref}

\title{Noise-Debiased Thermodynamic Variance for Local Learning Coefficient Probes}

\author{Yingjia Cai}
\date{June, 2026}

\newtheorem{proposition}{Proposition}
\newtheorem{theorem}{Theorem}

\begin{document}

\maketitle

\begin{abstract}
Local learning coefficient (LLC) probes offer a singularity-aware view of neural-network training, but mean-energy methods require a local loss baseline that is ambiguous at transient checkpoints. Thermodynamic variance avoids this input; under mini-batch evaluation, however, direct variance mixes cross-state loss fluctuations with same-state noise. We operationalize this route with the \emph{Shift-Invariant Variance Estimator} (SIVE), which estimates and subtracts the latter component using repeated evaluations. Conditional on any fixed retained path, unclipped SIVE is unbiased for noiseless path variance without requiring MCMC stationarity. The finite-scale diagnostic remains indexed by localization scale $h$---even a locally linear loss has tether-dependent variance---while interpretation as a Real Log Canonical Threshold (RLCT) requires additional stationary low-temperature conditions. Toy experiments recover calibrated finite-scale targets. At the primary localization scale, all five MNIST MLP trajectories exhibit a mid-training trough followed by a rebound in SIVE, while Raw Variance decreases from Epoch 40 to 100 in every trajectory. Across four localization scales, the joint early-drop/late-rise criterion is met in 19 of 20 trajectory--scale pairs. At Epoch 40, the estimated observation-noise correction accounts for $77.5\%$ of Raw Variance. Same-state debiasing thus reveals a reproducible turning structure masked by time-varying observation noise.
\end{abstract}

\section{Introduction}

Training loss and accuracy summarize performance at a checkpoint but not how the loss responds to local parameter fluctuations. Such fluctuations may expose training-regime changes hidden by standard curves. In a stochastic pipeline, however, an apparent fluctuation has two sources: movement through parameter space and repeated evaluation of the same state on different mini-batches.

Singular Learning Theory (SLT) motivates local thermodynamic probes of neural-network loss landscapes through the Real Log Canonical Threshold (RLCT), commonly called a Local Learning Coefficient (LLC) in empirical work. Recent studies have used LLC trajectories to identify stagewise transformer development and refined LLCs to follow component specialization \citep{hoogland2025loss,wang2025differentiation}, making reliable longitudinal measurement a practical concern. Mean-energy probes conventionally reference the loss at a local solution; Imai's thermodynamic-variance route removes this reference and provides our starting point \citep{imai2019estimating}.

Deploying this route at transient checkpoints creates a separate measurement problem: finite localized SGLD paths are evaluated on mini-batches, so directly computed thermodynamic variance mixes loss fluctuations across visited parameter states with same-state evaluation noise. SIVE identifies and subtracts the latter through replicated evaluations, turning the inherited variance route into a scale-indexed training diagnostic. The estimator targets fixed-path noiseless variance; links to thermodynamic and RLCT limits remain scale- and sampling-dependent.

Our main contributions are:
\begin{itemize}[leftmargin=2em]
    \item \textbf{Observation-noise correction:} Before optional clipping, SIVE uses same-state replicated evaluations to subtract the observation component and is conditionally unbiased for the fixed-path noiseless variance.
    \item \textbf{Scale-aware interpretation:} We distinguish finite-path, finite-scale thermodynamic, and stationary LLC/RLCT targets, and make the $h$-dependent non-stationary slope contribution explicit.
    \item \textbf{Controlled validation:} Toy calibrations isolate anchoring and observation-noise errors; five neural trajectories and a four-scale audit quantify the practical effect of correction on longitudinal probing.
\end{itemize}

\section{Background and Probe Setup}

\paragraph{RLCT and Local Thermodynamic Probes.}
Standard statistical learning theory assumes regular models with identifiable parameters and non-degenerate local geometry. Deep neural networks are instead singular \citep{watanabe2009algebraic, wei2022deep}: distinct parameter values can encode the same function, producing degenerate low-loss regions. In SLT, the Real Log Canonical Threshold (RLCT) $\lambda$ replaces ordinary parameter counting in the leading logarithmic complexity term of the asymptotic free energy and equivalently describes the local scaling of low-loss volume. Empirical local analyses commonly call the corresponding local quantity a Local Learning Coefficient (LLC). Computational studies probe it by simulating a localized tempered distribution, often with Stochastic Gradient Langevin Dynamics (SGLD) \citep{welling2011bayesian, lau2023local}. We reserve $\lambda$ for the stationary low-temperature target. Because neural-network training is a non-equilibrium process \citep{mandt2017stochastic, chaudhari2018stochastic} and a mid-training checkpoint need not be a local minimizer, the quantities measured along training are first treated as finite-path or finite-scale thermodynamic observables.

\paragraph{Mean-Energy Route and Anchoring.}
For a local chart defined by an energy $K(\theta)=L(\theta)-L^\star$ above a local optimum $L^\star$, the leading-order statistical-mechanical identity yields the mean-energy estimator
\begin{equation}
    \hat{\lambda}_{\mathrm{mean}} \approx t \, \mathbb{E}_{q}[K(\theta)] = t \, \mathbb{E}_{q}[L(\theta)-L^\star],
\end{equation}
where $t$ is the chosen inverse-temperature scale (often written $t=n\beta$) and $q$ is the localized tempered posterior.

At a reference local minimum $\theta_{t_k}$, empirical LLC estimation uses the loss of the reference parameter \citep{lau2023local}, giving the checkpoint-anchored response
\begin{equation}
    \lambda_{\mathrm{ckpt}}^{\mathrm{mean}}
    =t\left(\mathbb E_q[L(\theta)]-L(\theta_{t_k})\right).
    \label{eq:checkpoint_mean_response}
\end{equation}
This quantity is directly computable at any checkpoint, but away from a local minimum it is a finite-scale response rather than an LLC estimate: a localized path may have lower mean loss than its center. The neural experiment reports this conventional checkpoint reference together with path-minimum and final-checkpoint reference choices. These three curves share the same mean-energy statistic and differ only in the reference loss; they expose anchoring sensitivity rather than constituting three independent thermodynamic identities.

\paragraph{Thermodynamic-Variance Route and Floor Independence.}
Imai's variance route removes the reference value rather than estimating it more accurately. Let $\ell_n(\theta)=\log p(X^n\mid\theta)$ denote the total log likelihood and let
\begin{equation}
    \pi_\beta(\theta)\propto \exp\{\beta\ell_n(\theta)\}\varphi(\theta)
\end{equation}
be its power posterior. Its log normalizer
\begin{equation}
    A_n(\beta)
    =\log\int \exp\{\beta\ell_n(\theta)\}\varphi(\theta)\,d\theta
\end{equation}
satisfies the standard thermodynamic identities
\begin{equation}
    A_n'(\beta)=\mathbb E_{\pi_\beta}[\ell_n(\theta)],
    \qquad
    A_n''(\beta)=\operatorname{Var}_{\pi_\beta}[\ell_n(\theta)].
\end{equation}
Thus variance is the second thermodynamic response of the same log normalizer, rather than an ad hoc replacement for mean energy. For $\beta=c/\log n$, \citet{imai2019estimating} proposed
\begin{equation}
    \widehat\lambda_{\mathbb V}
    =\beta^2\operatorname{Var}_{\pi_\beta}[\ell_n(\theta)]
    \label{eq:imai_variance}
\end{equation}
and proved consistency for the RLCT under the corresponding SLT assumptions. With the average negative log likelihood $L_n(\theta)=-\ell_n(\theta)/n$ and $t=n\beta$, Eq.~\eqref{eq:imai_variance} is exactly
\begin{equation}
    \widehat\lambda_{\mathbb V}=t^2\operatorname{Var}_{\pi_\beta}[L_n(\theta)].
\end{equation}
This translation from Imai's notation to $t^2\operatorname{Var}(L_n)$ is exact; the asymptotic SLT assumptions enter only when the quantity is identified with $\lambda$. Because variance is invariant to additive shifts, the route requires neither $L^\star$ nor an estimate of it. Its thermodynamic origin, RLCT consistency under Imai's assumptions, and floor independence make it attractive for transient local probing. SIVE adopts and preserves these inherited properties; its new contribution concerns the stochastic observation layer introduced in deployment.

\paragraph{Localized Checkpoint Probe and Dual Floor Elimination.}
Let $L(\theta)$ denote the reference loss and let $\theta_{t_k}$ be a training checkpoint, not necessarily a local minimizer. We define a local chart with the Gaussian tether
\begin{equation}
    \phi_k(\theta) \propto \exp\left\{-\frac{\|\theta-\theta_{t_k}\|^2}{2h^2}\right\}.
\end{equation}
For an intended local region $U_k$, write
\begin{equation}
    L_k^\star=\inf_{\theta\in U_k}L(\theta),
    \qquad
    K_k(\theta)=L(\theta)-L_k^\star.
\end{equation}
At inverse-temperature scale $t$, the normalized localized law satisfies
\begin{equation}
\begin{aligned}
    q_{k,t,h}(\theta)
    &=\frac{\exp\{-t(L(\theta)-L_k^\star)\}\phi_k(\theta)}
    {\int \exp\{-t(L(\theta')-L_k^\star)\}\phi_k(\theta')\,d\theta'} \\
    &=\frac{\exp\{-tL(\theta)\}\phi_k(\theta)}
    {\int \exp\{-tL(\theta')\}\phi_k(\theta')\,d\theta'}.
\end{aligned}
\end{equation}
Thus the unknown floor cancels from the sampling law, and the localized SGLD drift can use the ordinary loss gradient. It cancels independently from the variance measurement:
\begin{equation}
    \operatorname{Var}_{q_{k,t,h}}[K_k(\theta)]
    =\operatorname{Var}_{q_{k,t,h}}[L(\theta)-L_k^\star]
    =\operatorname{Var}_{q_{k,t,h}}[L(\theta)].
\end{equation}
The floor is therefore required neither to run the localized chain nor to compute thermodynamic variance. The Gaussian tether discourages an off-equilibrium probe from drifting away from its checkpoint, but it is a soft localization with full support; identifying its low-temperature limit with a local RLCT additionally requires concentration in the intended region. Its radius $h$ remains part of the finite-scale measurement.

\paragraph{Stochastic Deployment Gap.}
Imai's estimator is formulated using the full-data log likelihood, which is deterministic conditional on $\theta$. A practical stochastic neural-network probe instead observes mini-batch losses. For $\theta$ distributed according to the probe law, write $\widehat L(\theta,\xi)=L(\theta)+\epsilon(\theta,\xi)$ with $\mathbb E[\epsilon\mid\theta]=0$. The population identity is
\begin{equation}
    \operatorname{Var}[\widehat L(\theta,\xi)]
    =\operatorname{Var}[L(\theta)]
    +\mathbb E\!\left[\operatorname{Var}(\epsilon\mid\theta)\right]
    \label{eq:background_total_variance}
\end{equation}
This identity shows that additive-floor independence does not remove random within-state variation. Subtracting a within-group variance estimate from the between-group variation is the equal-replication form of a classical ANOVA variance-component construction \citep{searle1992variance}; the same form appears in nested simulation estimators of $\operatorname{Var}(\mathbb E[X\mid Z])$ \citep{zouaoui2003accounting,sun2011efficient}. Here we apply this standard construction conditionally to a fixed retained SGLD path, targeting its realized noiseless sample variance rather than an outer-population variance. Applying thermodynamic variance directly to stochastic group means mixes parameter-state fluctuations with evaluation noise, both amplified by $t^2$. Averaging independently generated datasets may reduce estimator variance but does not identify this same-state component for a given deployment. SIVE retains the variance route's independence from an unknown loss floor and uses same-state replicated evaluations to identify and subtract the observation component. The next section states the resulting finite-path target and separates this estimation guarantee from sampler convergence and the stronger stationary LLC/RLCT interpretation.

\section{SIVE: Finite-Path Estimation}
\subsection{Finite-Path Target and Scale}
\label{sec:target_hierarchy}

We distinguish the computed estimator introduced below, its fixed-path target, a finite-scale thermodynamic observable, and the stationary RLCT limit. Let $\mathcal W\subseteq\{1,\ldots,M_{\mathrm{tot}}\}$ be a retention rule fixed before inspecting the probe losses (for example, a post-burn-in tail), and write $M_{\mathcal W}=|\mathcal W|>1$. For values $x_m$ indexed by $\mathcal W$, let $S_{\mathcal W}^2(x)$ denote their sample variance with denominator $M_{\mathcal W}-1$. Given the retained parameter path $\theta_{\mathcal W}$, define the noiseless path fluctuation
\begin{equation}
    \Lambda_{\mathcal W}^{\mathrm{path}}
    = t^2 S_{\mathcal W}^2\big(L(\theta_m):m\in\mathcal W\big).
    \label{eq:path_target}
\end{equation}
The empirical SIVE statistic $\widehat\Lambda_{\mathrm{SIVE}}$ estimates this finite-path target from noisy evaluations. Under mixing, a growing retained path, and vanishing sampler-discretization error, the path target approaches the finite-scale thermodynamic observable
\begin{equation}
    \Lambda_{k,t,h}^{\mathrm{therm}}
    = t^2\operatorname{Var}_{q_{k,t,h}}[L(\theta)].
    \label{eq:thermodynamic_target}
\end{equation}
Only when the checkpoint defines a stationary local chart, the tempered law remains concentrated in the intended local region as $t\to\infty$, and the usual low-temperature SLT assumptions hold do we invoke
\begin{equation}
    \Lambda_{k,t,h}^{\mathrm{therm}} = \lambda_k + o(1),
    \qquad t\to\infty.
    \label{eq:rlct_limit}
\end{equation}
Our checkpoint-level guarantee therefore concerns Eq.~\eqref{eq:path_target}; the subsequent links require the mixing, concentration, and low-temperature conditions stated above.

Floor cancellation does not imply slope invariance. Let $x=\theta-\theta_{t_k}$ and consider the exactly linear local loss $L(\theta_{t_k}+x)=L(\theta_{t_k})+g_k^\top x$, whose Hessian is zero. Under the Gaussian tether,
\begin{equation}
    q_{k,t,h}(x)\propto
    \exp\left\{-t g_k^\top x-\frac{\|x\|^2}{2h^2}\right\},
    \qquad
    x\sim\mathcal N(-t h^2g_k,h^2I),
\end{equation}
and hence
\begin{align}
    t\left(\mathbb E_{q_{k,t,h}}[L]-L(\theta_{t_k})\right)
    &=-t^2h^2\|g_k\|^2,
    \label{eq:linear_checkpoint_mean}\\
    t^2\operatorname{Var}_{q_{k,t,h}}[L]
    &=t^2h^2\|g_k\|^2.
    \label{eq:linear_tether_main}
\end{align}
Thus the same curvature-free non-stationary slope makes the checkpoint-anchored mean response negative and the variance response positive, both at a scale set by $h$. The localization radius is part of what the finite probe measures, not merely a numerical stabilizer. Appendix~\ref{app:gaussian_tether_slope} gives the corresponding curvature--slope decomposition.

\subsection{Observation-Noise Debiasing}

The thermodynamic variance target is formulated using the full-data log likelihood. In a stochastic deep-learning implementation, replacing it by mini-batch evaluations introduces a separate observation layer. We operationally decouple that layer from the SGLD parameter dynamics. Let $B_{\mathrm{grad}}$ denote the mini-batch size used to compute stochastic gradients for the Langevin drift, and $B_{\mathrm{eval}}$ denote the mini-batch size used to evaluate the loss at a given state.

For a single evaluation mini-batch of size $B_{\mathrm{eval}}$ at state $\theta_m$, define the noisy observation as
\begin{equation}
    \hat L_{m,i} = L(\theta_m) + \varepsilon_{m,i}, \qquad \mathbb{E}[\varepsilon_{m,i}\mid \theta_m] = 0, \quad \operatorname{Var}(\varepsilon_{m,i}\mid \theta_m) = \sigma_{\mathrm{eval}}^2(\theta_m,B_{\mathrm{eval}}).
\end{equation}
At the population level, the Law of Total Variance gives
\begin{equation}
    \operatorname{Var}_{q,\mathcal{B}}[\hat L_{m,i}] = \operatorname{Var}_{q}[L(\theta)] + \mathbb{E}_{q}[\sigma_{\mathrm{eval}}^2(\theta,B_{\mathrm{eval}})].
\end{equation}

By drawing $N$ conditionally independent evaluation mini-batches at the same state $\theta_m$, we compute the empirical mean over the evaluation group:
\begin{equation}
    \bar L_m = \frac{1}{N}\sum_{i=1}^N \hat L_{m,i}.
\end{equation}
Because the evaluations are conditionally independent given $\theta_m$, the observation variance is reduced by a factor of $N$, yielding:
\begin{equation}
    \operatorname{Var}_{q,\mathcal{B}}[\bar L_m] = \operatorname{Var}_{q}[L(\theta)] + \mathbb{E}_{q}\left[ \frac{\sigma_{\mathrm{eval}}^2(\theta,B_{\mathrm{eval}})}{N} \right].
\end{equation}

Relying on the uncorrected empirical variance of $\bar L_m$ would still scale the residual observation noise by $t^2$. We debias this by computing the sample variance within each group to estimate the state-dependent noise profile:
\begin{equation}
    s_m^2 = \frac{1}{N-1}\sum_{i=1}^{N}\left(\hat{L}_{m,i} - \bar{L}_m\right)^2.
\end{equation}
Specializing this between/within decomposition to the retained path gives the estimator's primary guarantee.

\begin{theorem}[Finite-path conditional unbiasedness]
\label{thm:finite_path_unbiasedness}
Let $M_{\mathcal W},N>1$. Conditional on a fixed retained path $\theta_{\mathcal W}$, suppose the evaluation errors $\{\varepsilon_{m,i}:m\in\mathcal W,1\le i\le N\}$ are independent, centered, and may have state-dependent variances $\sigma_m^2$. The evaluation draws must be fresh and independent of the transition randomness: neither they nor any functions of them are used to generate the parameter transitions. Then
\begin{equation}
\mathbb{E}\!\left[
S_{\mathcal W}^2(\bar L)
-\frac{1}{M_{\mathcal W}}\sum_{m\in\mathcal W}\frac{s_m^2}{N}
\,\middle|\,\theta_{\mathcal W}
\right]
=S_{\mathcal W}^2\big(L(\theta_m):m\in\mathcal W\big).
\label{eq:finite_path_unbiasedness}
\end{equation}
\end{theorem}

\begin{proof}
Conditional on the path, $\operatorname{Var}(\bar L_m\mid\theta_m)=\sigma_m^2/N$. Direct expansion of the heteroscedastic sample variance gives
$\mathbb E[S_{\mathcal W}^2(\bar L)\mid\theta_{\mathcal W}]
=S_{\mathcal W}^2(L)+(M_{\mathcal W})^{-1}\sum_{m\in\mathcal W}\sigma_m^2/N$.
Because $\mathbb E[s_m^2\mid\theta_m]=\sigma_m^2$, the subtraction cancels the second term exactly.
\end{proof}

Accordingly, SIVE is the \emph{unclipped}, observation-noise-debiased statistic
\begin{equation}
    \widehat\Lambda_{\mathrm{SIVE}}^{\mathrm{unclip}}
    =t^2\left(
    S_{\mathcal W}^2(\bar L)
    -\frac{1}{M_{\mathcal W}}\sum_{m\in\mathcal W}\frac{s_m^2}{N}
    \right),
    \label{eq:debiased-variance}
\end{equation}
which targets Eq.~\eqref{eq:path_target}. Regardless of whether the retained path is assigned an equilibrium or LLC interpretation, Theorem~\ref{thm:finite_path_unbiasedness} identifies its noiseless cross-state loss variance from stochastic evaluations. SLT enters only in the subsequent relation between that path target, the finite-scale thermodynamic observable, and the asymptotic RLCT. For numerical reporting one may optionally use
\begin{equation}
    \widehat\Lambda_{\mathrm{SIVE}}^{\mathrm{clip}}
    =\left(\widehat\Lambda_{\mathrm{SIVE}}^{\mathrm{unclip}}\right)_+.
\end{equation}
Clipping enforces nonnegativity but introduces an upward finite-sample bias; the distinction is material when $N$ is small. The estimator requires $N$ independent evaluation mini-batches only at the retained states in $\mathcal W$, creating a direct precision--computation tradeoff. The reported implementation records these groups at every SGLD step and applies burn-in afterward; this uniform recording scheme is a computationally more expensive superset of the estimator's requirement.

The debiasing correction above acts only on the observation layer: it removes the conditional variance introduced by stochastic loss evaluations given the visited SGLD states. It does not correct bias in the sampling distribution itself, such as finite-step discretization error or stochastic-gradient-induced perturbations of the Langevin drift. In the deep-learning experiments, we therefore interpret SIVE as a stochastic-loss-compatible fixed-step local probe rather than an asymptotically exact posterior sampler.

\begin{algorithm}[t]
\caption{SIVE: Observation-Noise Debiasing on a Localized Finite Path}
\label{alg:new_estimator}
\begin{algorithmic}[1]
\STATE \textbf{Input:} Checkpoint $\theta_{t_k}$; inverse-temperature scale $t$; localization radius $h$; SGLD step size $\eta$; total SGLD steps $M_{\mathrm{tot}}$; prespecified retained set $\mathcal W$; evaluation group size $N$; evaluation mini-batch size $B_{\mathrm{eval}}$; gradient mini-batch size $B_{\mathrm{grad}}$
\STATE Initialize $\theta^{(0)}=\theta_{t_k}$
\FOR{MCMC step $m=0$ to $M_{\mathrm{tot}}-1$}
    \IF{$m\in\mathcal W$}
        \STATE Draw $N$ independent evaluation mini-batches $\mathcal B^{\mathrm{eval}}_{m,1},\ldots,\mathcal B^{\mathrm{eval}}_{m,N}$, each of size $B_{\mathrm{eval}}$, and compute $\hat L_{m,i}=\hat L_{\mathcal B^{\mathrm{eval}}_{m,i}}(\theta^{(m)})$
        \STATE Store $\bar L_m=N^{-1}\sum_{i=1}^N\hat L_{m,i}$ and $s_m^2=(N-1)^{-1}\sum_{i=1}^N(\hat L_{m,i}-\bar L_m)^2$
    \ENDIF
    \STATE Compute stochastic gradient $g_m=\nabla\hat L(\theta^{(m)})$ using a mini-batch of size $B_{\mathrm{grad}}$ from a random stream independent of evaluation
    \STATE Draw $\xi_m\sim\mathcal N(0,I)$ from a third independent random stream
    \STATE Update $\theta^{(m+1)}=\theta^{(m)}-\eta t g_m-\eta(\theta^{(m)}-\theta_{t_k})/h^2+\sqrt{2\eta}\xi_m$
\ENDFOR
\STATE Compute $\widehat\Lambda_{\mathrm{SIVE}}^{\mathrm{unclip}}$ from the stored groups using Eq.~\eqref{eq:debiased-variance}; optionally report its positive part separately
\STATE \textbf{Output:} Finite-path SIVE estimate and optional clipped diagnostic
\end{algorithmic}
\end{algorithm}

\section{Experiments}
We evaluate the methodology in two stages: first, on analytically tractable toy models that isolate estimator error mechanisms under controlled conditions; second, by tracking deep-network SGD trajectories in a realistic off-equilibrium setting.

\subsection{Controlled Validation on Toy Models}
\label{sec:controlled_toy}

\subsubsection{Experimental Setup}

We consider two analytically tractable local energy models: 
\begin{align}
\text{Multiplicity }m=1: \qquad &K_1(u) = u^2, \\
\text{Multiplicity }m=2: \qquad &K_2(u,v) = u^2v^2.
\end{align}
For both, $\min_\theta K(\theta)=0$. The observed empirical loss is deliberately shifted by a hidden constant basin floor $L_0$ and corrupted by independent evaluation noise $\epsilon \sim \mathcal{N}(0, \sigma^2)$:

\begin{equation}
    \hat L(\theta) = L_0 + K(\theta) + \epsilon.
\end{equation}

We write $\theta=u$ for $K_1$ and $\theta=(u,v)$ for $K_2$, and use $K(\theta)$ to denote either energy when no distinction is needed. This formulation cleanly decouples three components: the true geometric signal $K(\theta)$, the unknown basin floor $L_0=L^\star$, and the stochastic observation noise $\epsilon$. Both models share an asymptotic continuous-time RLCT of $\lambda = 1/2$, though the $m=2$ singularity introduces finite-temperature corrections that shift the practical target below $1/2$ at $t=10^4$ (see Appendix~\ref{app:finite_temperature_m2}).

We simulate localized SGLD with inverse-temperature scale $t=10^4$ and localization radius $h=2$. To strictly isolate evaluation noise from dynamic optimization noise, SGLD transitions are driven by exact analytical gradients rather than stochastic mini-batches. At each sampled state $\theta_m$, we draw $N=64$ independent noisy observations $\hat{L}_{m,i}$ to compute the group mean $\bar{L}_m$ and within-group sample variance $s_m^2$. The SGLD learning rate is parameterized as $\eta=\texttt{base\_lr}/t$ with $\texttt{base\_lr}=0.05$. Full parameter explorations are detailed in Appendix~\ref{app:stable_operating_point}.

\subsubsection{Estimator Configurations}
We compare four distinct estimator configurations to isolate the impacts of minimization bias and evaluation noise:

\begin{enumerate}[label=\roman*., leftmargin=2em]
    \item \textbf{Oracle Mean-based:} Requires the exact analytical basin floor $L^\star=L_0$ and evaluates strictly on the \textit{noiseless} loss contour. This privileged end-to-end reference is unavailable in realistic off-equilibrium settings.
    \item \textbf{Naive Mean-based:} Does not require the true $L^\star$, instead substituting it with the empirical minimum along the noisy trajectory, $\hat{L}^\star = \min_m \hat{L}(\theta_m)$. It processes noisy loss evaluations without any noise removal, providing a deliberately implementable online anchoring baseline.
    \item \textbf{Raw Variance-based:} Implements the thermodynamic-variance estimator in Eq.~\eqref{eq:imai_variance} directly on noisy grouped losses. It bypasses $L^\star$ but contains the mini-batch evaluation variance and therefore represents the uncorrected stochastic deployment.
    \item \textbf{SIVE:} Retains the shift invariance of Raw Variance while subtracting the averaged within-group observation-noise profile in Eq.~\eqref{eq:debiased-variance}. The reported toy tables apply the optional positive-part safeguard; at small $N$ this distinction creates an upward clipping bias.
\end{enumerate}
The DNN runner also emits the uncorrected Raw Variance from the same retained histories at no additional sampling cost. Section~\ref{sec:dnn} uses it to audit the practical effect of the within-state noise correction.

\subsubsection{Results and Analysis}

\begin{table}[ht]
\centering
\caption{Finite-scale probe estimates on controlled noisy toy models. SIVE denotes the noise-debiased variance estimator with the positive-part numerical safeguard used in the reported runs.}
\label{tab:toy_results_main}
\begin{tabular}{lcc}
\toprule
Estimator Configuration
& $m=1$: $K(u)=u^2$
& $m=2$: $K(u,v)=u^2v^2$ \\
\midrule
Oracle Mean-based
& $0.5275 \pm 0.0061$
& $0.3980 \pm 0.0424$ \\
Naive Mean-based
& $11.0155 \pm 0.4489$
& $10.9775 \pm 0.3132$ \\
Raw Variance-based
& $6.8263 \pm 0.0336$
& $6.6818 \pm 0.0919$ \\
\textbf{SIVE}
& $\mathbf{0.5755 \pm 0.0355}$
& $\mathbf{0.4310 \pm 0.0961}$ \\
\bottomrule
\end{tabular}
\end{table}

Table~\ref{tab:toy_results_main} reports the scaled estimates averaged over five independent trials. Both models have asymptotic RLCT $\lambda=1/2$, but their finite-scale observables differ from this limit. For the regular quadratic model, the finite-step ULA references are approximately $0.5263$ for the mean-energy observable and $0.5540$ for the variance observable. For the singular model with the actual Gaussian localization $(t,h)=(10^4,2)$ and chart center $(0.2,0.2)$, numerical quadrature gives continuous-time targets $0.4296$ for the mean observable and $0.4246$ for the variance observable. The remaining finite-step target is not available in closed form for the singular chain. Appendix~\ref{app:finite_temperature_m2} and Appendix~\ref{app:ULA} separate these calibration levels.

Evaluating against these baselines, the empirical failure modes are consistent with our theoretical error decomposition, fully exposing the vulnerabilities of uncorrected estimators:
\begin{itemize}[leftmargin=2em]
    \item \textbf{Minimization Bias (Naive Mean):} The naive mean estimator overestimates the geometric signal (yielding $\approx 11.0$). Because $\hat{L}^\star$ is the minimum over $M$ noisy evaluations along the trajectory, negative order-statistics inherently drive $\hat{L}^\star$ below the true $L^\star$. Even a microscopic downward shift $\Delta = L^\star - \hat{L}^\star \approx 10^{-3}$ induces an artificial error term $t\Delta \approx 10.0$ at $t=10^4$, which entirely dominates the geometric signal.

    \item \textbf{Noise Amplification (Raw Variance):} Raw Variance absorbs additive evaluation noise by the Law of Total Variance. Scaled by $t^2$, the theoretical noise injection in our setting is $t^2\sigma^2/N=10^8\times(0.002)^2/64=6.25$. This explains why the uncorrected estimator outputs $\approx 6.7$ despite its valid thermodynamic origin.
    \item \textbf{Signal Recovery (SIVE):} SIVE bypasses the anchored-mean failure through the inherited variance route and corrects its stochastic observation layer. Since $s_m^2$ estimates $\sigma^2$, the averaged correction $M^{-1}\sum_{m=1}^{M}t^2s_m^2/N$ targets the same artifact $t^2\sigma^2/N=6.25$. Subtracting it returns the two estimates to the finite-scale geometric regime without estimating $L^\star$. In particular, the singular-model estimate $0.4310\pm0.0961$ is close to the Gaussian-localized continuous-time variance target $0.4246$; finite-step and finite-chain effects remain part of the sampler error.

\end{itemize}
A full derivation of the minimization-bias and noise-amplification mechanisms is provided in Appendix~\ref{app:error_decomposition}.

\subsection{Tracking SGD Dynamics in Deep Neural Networks}
\label{sec:dnn}
\subsubsection{Experimental Setup}
We train five independently initialized 3-layer Multi-Layer Perceptrons (MLPs) with ReLU activations on the MNIST image classification dataset using Stochastic Gradient Descent (SGD), with training seeds $1,\ldots,5$. We save the initialization and every subsequent checkpoint through Epoch 100. At each checkpoint $\theta_{t_k}$, we initialize five independent localized SGLD probe trials from the saved parameters and report the resulting finite-scale local fluctuation statistics. The main trajectory uses relative localization multiplier $c_h=1$; a paired audit repeats Epochs $\{1,40,100\}$ at $c_h\in\{0.5,1,2,4\}$. Implementation details are given in Appendix~\ref{app:dnn_implementation}.

\subsubsection{Probe Configurations}
The five plotted probes form two families. Let $\bar L_{\mathcal W}=M_{\mathcal W}^{-1}\sum_{m\in\mathcal W}\bar L_m$. The three mean-energy curves share
\begin{equation}
    \widehat\Lambda_{\mathrm{mean}}(b)
    =t(\bar L_{\mathcal W}-b)
\end{equation}
and differ only in their reference loss:
\begin{align}
    b_{\mathrm{ckpt}}&=L(\theta_{t_k}),
    & b_{\mathrm{path}}&=\min_{m\in\mathcal W}\bar L_m,
    & b_{\mathrm{final}}&=\min_{m\in\mathcal W}\bar L_{m,K},
    \label{eq:dnn_mean_references}
\end{align}
where $K=100$ is the final checkpoint and $b_{\mathrm{final}}$ is held fixed across checkpoints within a probe trial. We call these the \emph{checkpoint-reference}, \emph{path-minimum}, and \emph{final-reference} mean probes. The last two coincide at Epoch 100 by construction.

The variance family instead uses
\begin{equation}
    \widehat\Lambda_{\mathrm{Raw}}
    =t^2S_{\mathcal W}^2(\bar L),
    \qquad
    \widehat\Lambda_{\mathrm{SIVE}}
    =\widehat\Lambda_{\mathrm{Raw}}-\hat C,
    \qquad
    \hat C=\frac{t^2}{M_{\mathcal W}}\sum_{m\in\mathcal W}\frac{s_m^2}{N}.
    \label{eq:dnn_variance_pair}
\end{equation}
Raw Variance and SIVE use the same retained paths and evaluations, so they are the direct correction comparison. The three mean-energy probes provide reference-sensitivity context, not ablations of SIVE. We report unclipped SIVE; none of the DNN estimates in this run is negative.

\begin{figure}[ht]
    \centering
    \includegraphics[width=1.0\textwidth]{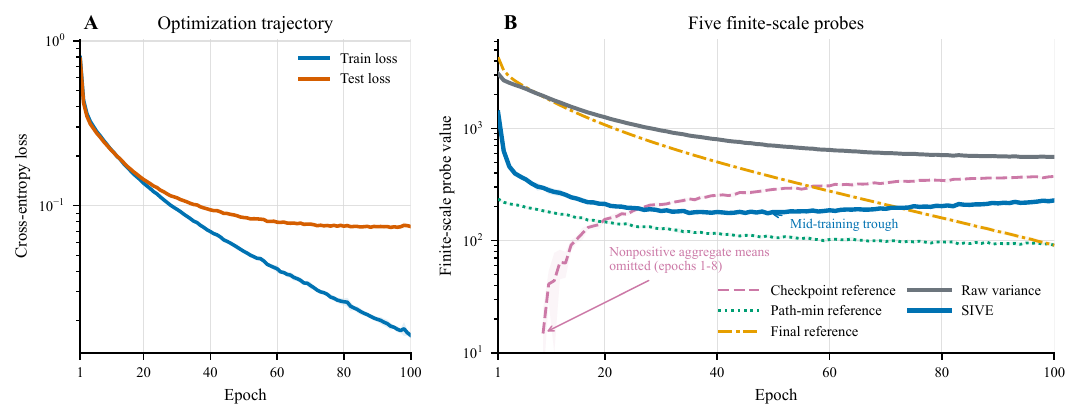}
    \caption{Five independently trained MNIST MLP trajectories, with five localized probe trials per checkpoint and $c_h=1$. Curves and bands show the mean $\pm$ sample standard deviation across the five training-trajectory-level probe means; Epoch 0 is omitted. Panel A shows train and test loss. Panel B shows all five probes on a positive logarithmic scale and marks the mid-training SIVE trough. Nonpositive aggregate checkpoint-reference means at Epochs 1--8 are omitted rather than clipped or reflected; complete signed values are reported in Appendix~\ref{app:dnn_five_trajectory_full}. Raw Variance and SIVE share the retained paths and evaluation groups; the three mean-energy probes differ only in reference loss but also change the statistic relative to the variance probes. Mean probes are compared for their longitudinal turning structure, not for numerical agreement with the variance probes.}
    \label{fig:mnist_dynamics}
\end{figure}

\subsubsection{Results and Analysis}
For each checkpoint, we first average the five probes within each training trajectory; means and sample standard deviations are then computed across the five independent trajectories, while temporal contrasts are paired within trajectory.

\begin{table}[ht]
\centering
\caption{Five-trajectory DNN audit at the prespecified checkpoints for $c_h=1$. Entries are mean $\pm$ sample standard deviation across five independently trained networks after first averaging five probe trials within each network.}
\label{tab:dnn_five_trajectory}
\resizebox{\textwidth}{!}{%
\begin{tabular}{rccccccc}
\toprule
& & & \multicolumn{3}{c}{Mean energy: reference choice} & \multicolumn{2}{c}{Thermodynamic variance} \\
\cmidrule(lr){4-6}\cmidrule(lr){7-8}
Epoch & Train Loss & Test Loss & Checkpoint & Path minimum & Final reference & Raw & SIVE \\
\midrule
1   & $0.7929\pm0.0393$ & $0.7741\pm0.0394$ & $-3161.1\pm314.7$ & $235.0\pm7.3$ & $4322.5\pm85.7$ & $3083.6\pm57.8$ & $1385.0\pm79.1$ \\
40  & $0.0696\pm0.0022$ & $0.0943\pm0.0020$ & $252.7\pm8.5$ & $116.8\pm3.6$ & $502.8\pm12.8$ & $800.9\pm15.9$ & $179.8\pm3.7$ \\
100 & $0.0163\pm0.0007$ & $0.0747\pm0.0012$ & $373.4\pm6.3$ & $90.9\pm4.1$ & $90.9\pm4.1$ & $558.8\pm15.5$ & $228.1\pm10.0$ \\
\bottomrule
\end{tabular}%
}
\end{table}

\paragraph{Observation-Noise Correction Changes the Inferred Turning Structure.}
All five SIVE trajectories decline to minima at Epochs 37--49 and then rebound, falling $86.1\%$--$88.2\%$ before rising $28.5\%$--$39.4\%$ by Epoch 100. Over Epochs 40--100, the paired changes are $-242.2\pm9.9$ for Raw Variance, $-290.5\pm10.7$ for the estimated correction, and $+48.3\pm8.5$ for SIVE; Raw falls and SIVE rises in every trajectory. The aggregate curves retain these opposite late directions for any intermediate checkpoint from Epoch 30 to 60 (Appendix~\ref{app:dnn_five_trajectory_full}). The correction accounts for $55.1\%$, $77.5\%$, and $59.2\%$ of Raw Variance at Epochs 1, 40, and 100, with the Epoch-40 trajectory-specific shares ranging only from $77.0\%$ to $78.1\%$. Thus the time-varying observation component hides the direction reversal in Raw Variance. Same-state replications identify this component directly; without additional assumptions, it is not recoverable from stochastic group means alone.

\paragraph{Mean-Energy Reference Sensitivity.}
The aggregate checkpoint-reference response crosses zero between Epochs 8 and 9 and rises over both intervals, including a paired late increase of $120.7\pm12.7$. The path-minimum and final-reference responses instead decline over both intervals, with late changes of $-25.9\pm4.9$ and $-411.9\pm9.0$, respectively, in every trajectory. None changes direction at $c_h=1$. These are comparisons of end-to-end diagnostic shapes rather than interchangeable estimators: the path minimum introduces an order statistic, while the final reference imports later optimization descent.

\paragraph{Trajectory and Scale Consistency.}
For $c_h\in\{0.5,1,2,4\}$, the Raw/SIVE late contrasts are respectively $(-345.9,+5.5)$, $(-242.2,+48.3)$, $(-8.3,+101.1)$, and $(+76.1,+53.4)$. Together with the early SIVE drop in all $20$ trajectory--scale pairs, the positive late contrast in $19$ of $20$ shows that the direction change is not an artifact of one aggregate curve. Linear residualization does not preserve it, limiting the interpretation to an $h$-indexed operational diagnostic rather than a stationary LLC or purely curvature-based quantity.

\section{Conclusion and Future Work}
Thermodynamic variance offers a theoretically grounded, baseline-free route for bringing LLC-inspired local analysis into the training process, but stochastic evaluation mixes cross-state fluctuations with same-state mini-batch noise. SIVE removes the identifiable observation component and is conditionally unbiased for finite-path noiseless variance, while sampler convergence and asymptotic RLCT interpretation remain separate questions. Across five MNIST MLP trajectories at the primary scale, training-dependent observation noise masked the mid-training trough and subsequent rebound revealed by SIVE; both the correction and the finite-scale signal remained $h$-dependent. These trajectory and scale checks concern a single MNIST MLP configuration; broader validation across models and tasks remains open. This makes SIVE a measurement tool for testing whether such stagewise changes align with delayed-generalization or grokking-like transitions \citep{power2022grokking,nanda2023progress}, and for exploring scale-aware checkpointing, optimization, and regularization on larger models.

\paragraph{Reproducibility Statement.}
https://github.com/yingjiacai/SIVE
It contains the SIVE implementation, experiment configurations, experiment and plotting scripts, and tests. Implementation details and the complete five-trajectory DNN results are provided in Appendices~\ref{app:dnn_implementation} and~\ref{app:dnn_five_trajectory_full}, respectively.

\section*{AI Use Statement}
Generative AI tools were used to help develop and critically review the conceptual framing; formulate, check, and edit mathematical claims, derivations, and proofs; propose and refine hypotheses; provide feedback on experimental methodology and design; assist with software implementation and review; interpret and present results; identify and examine related literature; and draft and edit the manuscript. Generative AI was not used to generate the reported experimental data. The author reviewed all AI-assisted material, checked citations and reported numerical results against the underlying sources and experiment records, and takes responsibility for the final text, claims, code, and artifacts produced with generative-AI assistance.

\bibliographystyle{plainnat}
\bibliography{refs}
\clearpage

\appendix
\section{Free-Energy Identities and Gaussian-Localized Finite-Temperature Calibration}
\label{app:free_energy_and_finite_temp}
\label{app:finite_temperature_m2}

This appendix separates the asymptotic RLCT from the finite-temperature target induced by the actual Gaussian reference measure. The distinction is essential for the multiplicity-$2$ toy model, whose convergence is logarithmically slow and explicitly depends on the localization scale $h$.

\subsection{Free-Energy Identities}
\label{app:free_energy_identities}

Let $K(\theta) \ge 0$ define a local energy function (typically the loss gap $L(\theta) - L^\star$), and let $\phi(\theta)\,d\theta$ denote a localized reference measure. For inverse-temperature scale $t$, define
\begin{equation}
    Z(t) = \int \exp\{-tK(\theta)\}\phi(\theta)\,d\theta,
\end{equation}
and the corresponding free energy is:
\begin{equation}
    F(t) = -\log Z(t).
\end{equation}
The associated localized tempered posterior is:
\begin{equation}
    q_t(\theta) = \frac{\exp\{-tK(\theta)\}\phi(\theta)}{Z(t)}.
\end{equation}

According to the standard asymptotic expansion in Singular Learning Theory (SLT), as $t \to \infty$, the local Laplace integral is dominated by the most singular region of the parameter space:
\begin{equation}
    Z(t) \asymp C\,t^{-\lambda}(\log t)^{m-1}, 
\end{equation}
where $\lambda$ is the Real Log Canonical Threshold (RLCT) and $m$ is the multiplicity of the singularity. Equivalently, the leading-order free energy is:
\begin{equation}
    F_{\mathrm{lead}}(t) = \lambda \log t - (m-1)\log\log t.
\end{equation}

By standard statistical-mechanical identities, the posterior mean and variance of the local energy correspond exactly to the first and second derivatives of the free energy with respect to the inverse temperature:
\begin{equation}
    \mathbb{E}_{q_t}[K(\theta)] = F'(t),
    \qquad
    \operatorname{Var}_{q_t}[K(\theta)] = -F''(t).
\end{equation}
Thus the mean- and variance-energy observables are distinct derivatives of the same thermodynamic system. Their common asymptotic coefficient does not make the mean and variance observables identical at finite $t$.

\subsection{Exact Centered-Gaussian Calibration for \texorpdfstring{$u^2v^2$}{u2v2}}
\label{app:finite_temp_mean}
\label{app:finite_temp_variance}

For $K(u,v)=u^2v^2$ and the normalized centered Gaussian reference
\begin{equation}
    \phi_h(u,v)=\frac{1}{2\pi h^2}
    \exp\!\left[-\frac{u^2+v^2}{2h^2}\right],
\end{equation}
direct integration gives
\begin{equation}
    Z_h(t)=\frac{e^z K_0(z)}{2\sqrt{\pi}\,h^2\sqrt t},
    \qquad z=\frac{1}{8th^4},
    \label{eq:bessel_partition}
\end{equation}
where $K_\nu$ is the modified Bessel function of the second kind. Writing $r(z)=K_1(z)/K_0(z)$ and differentiating Eq.~\eqref{eq:bessel_partition} yields the exact centered-reference observables
\begin{align}
    t\,\mathbb E[K]
    &=\frac12+z\bigl(1-r(z)\bigr), \\
    t^2\operatorname{Var}(K)
    &=\frac12+2z-zr(z)+z^2\bigl(1-r(z)^2\bigr).
    \label{eq:bessel_observables}
\end{align}
For small $z$, define
\begin{equation}
    A(t,h)=\log(16th^4)-\gamma,
\end{equation}
where $\gamma$ is Euler's constant. Then
\begin{align}
    t\,\mathbb E[K]
    &=\frac12-\frac{1}{A(t,h)}+o\!\left(A^{-1}\right), \\
    t^2\operatorname{Var}(K)
    &=\frac12-\frac{1}{A(t,h)}-\frac{1}{A(t,h)^2}
      +o\!\left(A^{-2}\right).
    \label{eq:h_dependent_asymptotics}
\end{align}
The finite-temperature correction is therefore controlled by $\log(th^4)$, not by $\log t$ alone. The localization radius is part of the finite-scale estimand even though the leading RLCT remains $1/2$.

The reported toy chain is centered at $(0.2,0.2)$ rather than exactly at the singularity. Applying deterministic quadrature to that actual shifted Gaussian reference with $(t,h)=(10^4,2)$ gives
\begin{equation}
    t\,\mathbb E[K]=0.42957,
    \qquad
    t^2\operatorname{Var}(K)=0.42461.
    \label{eq:shifted_gaussian_targets}
\end{equation}
These are continuous-time localized targets. The finite-step singular ULA target remains distinct, as discussed next. Finally, the exact shift identity
$\operatorname{Var}[K]=\operatorname{Var}[L-L^\star]=\operatorname{Var}[L]$
continues to hold for every reference measure and every finite $t$.
\section{Finite-Step ULA Bias and the Moment-Closure Barrier}
\label{app:ula_bias}
\label{app:ULA}

The continuous-time reference values derived in Appendix~\ref{app:free_energy_and_finite_temp} formally assume exact sampling from the target thermodynamic posterior. In practice, however, the finite-scale probes are driven by the discrete-time Unadjusted Langevin Algorithm (ULA). This appendix demonstrates how finite step sizes alter the stationary distribution, injecting a systematic discretization bias. We first separate the mathematically tractable regular case—where exact moment calculations are explicitly solvable—from the singular case, where structural non-linearities obstruct closed-form solutions.

\subsection{Finite-Step Bias in Regular Quadratic Geometries}
\label{app:ula_regular}

To isolate the precise mechanics of ULA discretization bias, we analyze the strictly regular quadratic toy model, $K_1(u)=u^2$. Because the gradient is linear ($\nabla K_1 = 2u$), the infinite hierarchy of moments collapses into a solvable algebraic system.

Let the dimensionless effective step size be denoted by $a=\eta t=\texttt{base\_lr}$, and let $c$ be the one-dimensional chart center. Including the same Gaussian tether used by the experiment, the discrete ULA update is the affine AR(1) process
\begin{equation}
    u_{m+1}=r u_m+\delta c+\sqrt{2\eta}\,\xi_m,
    \qquad
    r=1-2a-\delta,
    \qquad
    \delta=\frac{\eta}{h^2},
    \qquad \xi_m\sim\mathcal N(0,1).
\end{equation}
For $|r|<1$, its exact stationary mean and variance are
\begin{equation}
    \mu_u=\frac{\delta c}{1-r},
    \qquad
    V_u=\frac{2\eta}{1-r^2}.
\end{equation}

Because the stationary distribution is Gaussian, the exact discrete-time targets for the implemented tethered chain are
\begin{equation}
    t\,\mathbb{E}_{\mathrm{ULA}}[u^2]=t(V_u+\mu_u^2),
    \qquad
    t^2\operatorname{Var}_{\mathrm{ULA}}(u^2)
    =t^2(2V_u^2+4\mu_u^2V_u).
    \label{eq:tethered_regular_ula}
\end{equation}
For the reported $(t,h,c)=(10^4,2,0.2)$, Eq.~\eqref{eq:tethered_regular_ula} rounds to the values shown below. The tether and center corrections are smaller than the displayed precision, so the familiar untethered expressions $[2(1-a)]^{-1}$ and $[2(1-a)^2]^{-1}$ give the same four-decimal references. As $a\to0$, Eq.~\eqref{eq:tethered_regular_ula} approaches the continuous Gaussian-localized target, not an automatically untethered quantity.

Table~\ref{tab:ULA} compares empirical measurements against these exact ULA targets for the regular $K_1$ model. We intentionally evaluate a relatively large step size ($a=0.10$) to amplify the discretization artifact.

\begin{table}[ht]
\centering
\caption{Finite-step calibration on the regular quadratic toy model $K_1(u)=u^2$ at $t=10^4$. ``Exact Target'' is computed from the tethered affine-Gaussian chain in Eq.~\eqref{eq:tethered_regular_ula}; the remaining discrepancies are finite-chain and observation-estimation variation.}
\label{tab:ULA}
\begin{tabular}{lcc|cc}
\toprule
$\texttt{base\_lr}=a$
& Oracle Mean
& Exact Target
& SIVE
& Exact Target \\
\midrule
$0.05$
& $0.5275 \pm 0.0061$
& $0.5263$
& $0.5755 \pm 0.0355$
& $0.5540$ \\
$0.10$
& $0.5549 \pm 0.0040$
& $0.5555$
& $0.6307 \pm 0.0285$
& $0.6173$ \\
\bottomrule
\end{tabular}
\end{table}

The quantitative agreement is consistent with the estimator tracking the distribution of the discrete samples even when finite-step Langevin dynamics widen that distribution. It does not remove the finite-step change of target.

\subsection{The Moment-Closure Barrier in Singular Geometries}
\label{app:ula_singular_barrier}

For singular geometries, the corresponding elementary moment calculation does not close. To display the obstruction while retaining localization, consider a centered Gaussian tether and the $u$ update for $K_2(u,v)=u^2v^2$:
\begin{equation}
    u_{m+1}=r_0u_m-2\eta t u_mv_m^2+\sqrt{2\eta}\xi_m,
    \qquad r_0=1-\frac{\eta}{h^2}.
\end{equation}
Squaring and taking expectations at stationarity gives
\begin{equation}
    0=(r_0^2-1)\mathbb E[u^2]
      -4r_0\eta t\mathbb E[u^2v^2]
      +4\eta^2t^2\mathbb E[u^2v^4]+2\eta.
\end{equation}
Equivalently,
\begin{equation}
    r_0t\mathbb E[u^2v^2]
    =\frac12+\eta t^2\mathbb E[u^2v^4]
     -\frac{1-r_0^2}{4\eta}\mathbb E[u^2].
    \label{eq:singular_ula_moment_balance}
\end{equation}

In the continuous-step limit, Eq.~\eqref{eq:singular_ula_moment_balance} retains the finite-$h$ identity
$t\mathbb E[u^2v^2]=\tfrac12-\mathbb E[u^2]/(2h^2)$; only the subsequent low-temperature asymptotics approach the RLCT $1/2$. At finite step size, the target is additionally coupled to $\mathbb E[u^2v^4]$. An off-center tether, as used in the reported experiment, introduces further first and mixed moments rather than restoring closure.

Consequently, this low-order moment method does not yield a closed-form finite-step target for the singular chain. We therefore report controlled step-size sensitivity in Appendix~\ref{app:operational_diagnostics} and treat $\texttt{base\_lr}=0.05$ as the selected finite-run operating point rather than an exact calibration.

\section{Finite-Path Error Decomposition for Anchored Mean Estimators}
\label{app:error_decomposition}

This appendix gives an exact decomposition for a mean statistic anchored to the lowest noisy loss on the same finite path. It separates incomplete path coverage from an order-statistics effect without assigning either a universal finite-realization sign.

\subsection{General Error Decomposition}
\label{app:error_general}

Write the grouped loss at step $m$ as
\begin{equation}
    \bar L_m=L_k^\star+K_m+\bar\varepsilon_m,
\end{equation}
where $K_m=L(\theta_m)-L_k^\star\ge0$. Define
\begin{equation}
    \min_m\bar L_m=L_k^\star+g_M-\nu_M,
    \qquad
    g_M=\min_mK_m,
    \qquad
    \nu_M=g_M-\min_m(K_m+\bar\varepsilon_m).
    \label{eq:signed_nu_definition}
\end{equation}
The realized $\nu_M$ is signed. Nevertheless, because the minimum is a concave functional of the centered evaluation noise,
\begin{equation}
    \mathbb E[\nu_M\mid K_{1:M}]\ge0.
    \label{eq:nu_conditional_sign}
\end{equation}
Substitution gives the exact finite-path identity
\begin{equation}
\begin{aligned}
    \widehat\Lambda_{\mathrm{online}}
    &=t\left(\frac1M\sum_m\bar L_m-\min_m\bar L_m\right)\\
    &=t\left(\frac1M\sum_mK_m-g_M\right)
      +t\nu_M+t\frac1M\sum_m\bar\varepsilon_m.
\end{aligned}
\label{eq:unified_bias_law_core}
\end{equation}
After averaging over evaluation noise, the first term is the noiseless path thickness above its lowest visited state and the second is the expected noisy-minimum inflation. Replacing either by a stationary LLC expression requires the additional assumptions in Section~\ref{sec:target_hierarchy}.

\subsection{Noise-Dominated Toy Regime}
\label{app:error_noise_dominated}

The controlled low-dimensional paths repeatedly visit states close to the exact minimum, so the order-statistics effect dominates the anchored result. Since $K_m\ge g_M$,
\begin{equation}
    \nu_M\le-\min_m\bar\varepsilon_m.
\end{equation}
For independent Gaussian evaluation noise, the corresponding pure-noise extreme-value scale is
\begin{equation}
    \frac{\sigma}{\sqrt N}\sqrt{2\log M_{\mathrm{eff}}}.
\end{equation}
It is an approximate upper scale, not a general equality for $\mathbb E[\nu_M]$, because only states sufficiently near the path minimum can realize the pure-noise extreme. With $t=10^4$, $\sigma=0.002$, $N=64$, and $M_{\mathrm{eff}}\approx7\times10^4$, the scaled value is approximately $11.8$, consistent in magnitude with the empirical Naive Mean result $\approx11.0$ in Table~\ref{tab:toy_results_main}.

\subsection{Path-Relative Anchoring in High Dimensions}
\label{app:error_floor_miss}

In a high-dimensional non-stationary chart, a finite localized path need not reach the unknown basin floor. Anchoring to its lowest visited loss therefore reports the path-relative thickness in Eq.~\eqref{eq:unified_bias_law_core}, rather than an independently calibrated thermodynamic target. The compressed path-minimum curve in Figure~\ref{fig:mnist_dynamics} is consistent with this interpretation. Without access to $L_k^\star$, however, neither the magnitude of $g_M$ nor a universal direction of error relative to an intrinsic LLC can be identified from the experiment alone.

\subsection{Capabilities and Limitations of SIVE}
\label{app:variance_limitations}

\begin{itemize}[leftmargin=2em]
    \item \textbf{Inherited shift invariance:} Thermodynamic variance removes the additive baseline $L_k^\star$ and never performs a noisy minimum operation; SIVE preserves this property.
    \item \textbf{Conditional observation-noise debiasing:} Theorem~\ref{thm:finite_path_unbiasedness} removes the expected variance contributed by conditionally independent mini-batch evaluations, including state-dependent noise levels.
    \item \textbf{Unaddressed path and sampler effects:} SIVE does not correct finite-step discretization, finite-chain Monte Carlo error, stochastic-gradient perturbations, or the dependence of the finite-scale estimand on $(t,h)$.
    \item \textbf{No slope invariance:} At a non-stationary checkpoint, parameter fluctuations along a local slope induce loss fluctuations even though the additive floor has canceled. The exact quadratic decomposition is given next.
\end{itemize}

\section{Localization Radius, Chain Length, and Tail-Window Consistency}
\label{app:h_m_sensitivity}

While the variance operator structurally eliminates the absolute baseline $L_k^\star$, the localized SGLD probe remains dependent on the scale of the local chart (indexed by $h$) and the convergence properties of the Markov chain (controlled in part by the total steps $M_{\mathrm{tot}}$). This appendix separates these two roles and makes explicit how the localized posterior interacts with off-equilibrium gradients.

\subsection{A Local Gaussian Tether Model on a Non-Equilibrium Slope}
\label{app:gaussian_tether_slope}

In practical deep learning, stochastic gradient descent operates as a non-equilibrium process; mid-training checkpoints $\theta_{t_k}$ are rarely true stationary points ($\nabla L \neq 0$). Consequently, localized MCMC chains explore gradient ``slopes'' rather than symmetric minima. 

To model this analytically, we approximate the local excess loss by a surrogate that explicitly retains both its checkpoint offset and the non-vanishing linear drift:
\begin{equation}
    K_k(\theta) \approx K_{k,0}+g_k^\top x + \frac{1}{2}x^\top H_k x,
    \qquad x = \theta - \theta_{t_k},
\end{equation}
where $K_{k,0}=L(\theta_{t_k})-L_k^\star$, $g_k=\nabla L(\theta_{t_k})$, and $H_k$ is the local Hessian. The offset $K_{k,0}$ does not affect the localized law, and it is removed exactly by a variance. The slope $g_k$, in contrast, changes the law and therefore remains part of the finite-scale estimand.

Under the Gaussian prior $\exp\{-\|x\|^2 / (2h^2)\}$, the tethered local Gibbs law becomes:
\begin{equation}
    q_{k,t,h}(x) \propto \exp\left\{-t\left(K_{k,0}+g_k^\top x + \frac{1}{2}x^\top H_k x\right) - \frac{1}{2h^2}\|x\|^2\right\}.
\end{equation}
Let $A_{k,t,h}=tH_k+h^{-2}I$. If $A_{k,t,h}\succ0$, this quadratic surrogate defines the Gaussian $\mathcal{N}(\mu_{k,t,h},\Sigma_{k,t,h})$ with
\begin{equation}
    \Sigma_{k,t,h}=A_{k,t,h}^{-1}, \qquad
    \mu_{k,t,h}=-t\Sigma_{k,t,h}g_k.
\end{equation}
The positive-definiteness condition is essential: a Gaussian tether does not make an indefinite quadratic approximation normalizable for arbitrary $h$. The original nonlinear localized law may nevertheless be normalizable outside the range where this local quadratic surrogate is valid.

\subsection{Curvature--Slope Decomposition of the Finite-Scale Variance}
\label{app:slope_plateau}

The following proposition makes explicit what shift invariance does---and does not---remove.

\begin{proposition}[Exact quadratic tether decomposition]
\label{prop:slope_plateau}
Assume $H_k$ is symmetric and $A_{k,t,h}\succ0$. Under the Gaussian quadratic surrogate,
\begin{equation}
\begin{split}
    \mathbb E[K_k]
    ={}&K_{k,0}-t g_k^\top\Sigma_{k,t,h}g_k
    +\frac{t^2}{2}g_k^\top\Sigma_{k,t,h}H_k\Sigma_{k,t,h}g_k\\
    &+\frac12\operatorname{tr}(H_k\Sigma_{k,t,h}),
\end{split}
\end{equation}
whereas the scaled variance is
\begin{equation}
    t^2\operatorname{Var}(K_k)
    =\frac12\operatorname{tr}\!\left[(tH_k\Sigma_{k,t,h})^2\right]
    +t^2h^{-4}g_k^\top\Sigma_{k,t,h}^{3}g_k.
    \label{eq:quadratic_tether_variance}
\end{equation}
For a tight tether, $\Sigma_{k,t,h}=h^2I+\mathcal O(h^4)$, and therefore
\begin{align}
    t\mathbb E[K_k]
    &=tK_{k,0}-t^2h^2\|g_k\|^2
      +\frac{th^2}{2}\operatorname{tr}(H_k)+\mathcal O(h^4),\\
    t^2\operatorname{Var}(K_k)
    &=t^2h^2\|g_k\|^2+\mathcal O(h^4).
    \label{eq:tight_tether_slope}
\end{align}
\end{proposition}

\begin{proof}
For $x\sim\mathcal N(\mu,\Sigma)$ and symmetric $H$,
\begin{equation}
    \operatorname{Var}\!\left(g^\top x+\frac12x^\top Hx\right)
    =\frac12\operatorname{tr}[(H\Sigma)^2]
     +(g+H\mu)^\top\Sigma(g+H\mu).
\end{equation}
Substitution of $\mu=-t\Sigma g$ gives the stated mean. Because
$I-tH_k\Sigma_{k,t,h}=h^{-2}\Sigma_{k,t,h}$, the second variance term becomes
$h^{-4}g_k^\top\Sigma_{k,t,h}^3g_k$. Finally, expanding
$\Sigma_{k,t,h}=h^2(I-th^2H_k+\mathcal O(h^4))$ proves the tight-tether expressions.
\end{proof}

The flat-slope case recovers Eq.~\eqref{eq:linear_tether_main}: if $H_k=0$, then $t^2\operatorname{Var}(K_k)=t^2h^2\|g_k\|^2$ exactly. Thus the variance removes the unknown additive floor but does not remove non-stationary slope effects. If $g_k=0$, only the curvature term in Eq.~\eqref{eq:quadratic_tether_variance} remains. Neither case supplies a universal $h$-independent plateau in singular or off-equilibrium directions; $h$ indexes the finite spatial scale being probed.

\subsection{Finite Chain Length and Tail Windows}
\label{app:tail_windows}

The second critical dependency is the total chain length $M_{\mathrm{tot}}$. Let $F_m$ be an observable evaluated along the SGLD chain (for the path target, $F_m=L(\theta_m)$), and let $\mu_j=\mathbb E_q[F^j]$ for $j\in\{1,2\}$. Controlling a variance statistic requires control of both coordinates of the observable $G_m=(F_m,F_m^2)$; consistency of the sample mean alone is insufficient.

To mitigate the transient drift, we define a tail window that discards the initial burn-in phase:
\begin{equation}
    \mathcal W_\alpha = \{\lfloor \alpha M_{\mathrm{tot}}\rfloor,\dots,M_{\mathrm{tot}}-1\}, \qquad 0<\alpha<1,
\end{equation}
and compute the empirical average:
\begin{equation}
    \overline{F^j}_{\mathcal W_\alpha} = \frac{1}{|\mathcal W_\alpha|} \sum_{m\in\mathcal W_\alpha} F_m^j,
    \qquad j\in\{1,2\}.
\end{equation}

\begin{proposition}[Tail mean and variance consistency]
\label{prop:finite_chain_consistency}
Assume $F$ has a finite fourth moment under $q$, the post-burn-in process $G_m=(F_m,F_m^2)$ is geometrically mixing, and the two transient moment biases satisfy
$|\mathbb E[F_m^j]-\mu_j|\le C_j\rho^m$ for $j\in\{1,2\}$ and some $0<\rho<1$. Let
$\tau_\star=\max\{\tau_{\mathrm{int}}(F),\tau_{\mathrm{int}}(F^2)\}$. Then, for $j\in\{1,2\}$,
\begin{equation}
    \overline{F^j}_{\mathcal W_\alpha}-\mu_j
    =\mathcal O(\rho^{\alpha M_{\mathrm{tot}}})
    +\mathcal O_p\!\left(|\mathcal W_\alpha|_{\mathrm{eff}}^{-1/2}\right),
\end{equation}
where
\begin{equation}
    |\mathcal W_\alpha|_{\mathrm{eff}}
    \asymp \frac{(1-\alpha)M_{\mathrm{tot}}}{\tau_\star}.
\end{equation}
Consequently,
\begin{equation}
    S_{\mathcal W_\alpha}^2(F)-\operatorname{Var}_q(F)
    =\mathcal O(\rho^{\alpha M_{\mathrm{tot}}})
    +\mathcal O_p\!\left(|\mathcal W_\alpha|_{\mathrm{eff}}^{-1/2}\right)
    +\mathcal O(|\mathcal W_\alpha|^{-1}).
\end{equation}
\end{proposition}

\begin{proof}
Apply the geometric-tail bias bound and the usual mixing variance bound separately to $F$ and $F^2$; the slower coordinate determines $\tau_\star$. The Bessel-corrected sample variance obeys
\begin{equation}
S_{\mathcal W_\alpha}^2(F)
=\frac{|\mathcal W_\alpha|}{|\mathcal W_\alpha|-1}
\left(\overline{F^2}_{\mathcal W_\alpha}
-\big(\overline F_{\mathcal W_\alpha}\big)^2\right).
\end{equation}
Substituting the two moment bounds and using $|\mathcal W_\alpha|=(1-\alpha)M_{\mathrm{tot}}+\mathcal O(1)$ gives the result.
\end{proof}

\subsection{Separating Scale Dependence from Estimation Error}
\label{app:combined_error_law}

The dependence on $h$ is part of the target, not a generic error term. For a tail-window implementation, a schematic decomposition is
\begin{equation}
    \widehat\Lambda_{\mathrm{tail}}^{\mathrm{unclip}}
    =\Lambda_{k,t,h}^{\mathrm{therm}}
    +b_{\mathrm{step}}+b_{\mathrm{grad}}
    +\mathcal O(\rho^{\alpha M_{\mathrm{tot}}})
    +\mathcal O_p\!\left(|\mathcal W_\alpha|_{\mathrm{eff}}^{-1/2}\right)
    +\varepsilon_{\mathrm{eval}},
    \label{eq:tail_error_decomposition}
\end{equation}
where $b_{\mathrm{step}}$ and $b_{\mathrm{grad}}$ denote discretization and stochastic-gradient biases. Proposition~\ref{prop:finite_chain_consistency} supplies the displayed Monte Carlo order for the variance statistic, with the effective sample size governed by the slower of $F$ and $F^2$. Conditional on a fixed retained path, Theorem~\ref{thm:finite_path_unbiasedness} gives $\mathbb E[\varepsilon_{\mathrm{eval}}\mid\theta_{\mathcal W_\alpha}]=0$. The remaining terms require sampler diagnostics or limits; no universal rate in $h$ follows from shift invariance.

The decomposition separates three conceptually different choices:
\begin{itemize}[leftmargin=2em]
    \item \textbf{Probe scale $(t,h)$:} changing either parameter generally changes $\Lambda_{k,t,h}^{\mathrm{therm}}$ itself.
    \item \textbf{Sampler approximation:} finite step size, stochastic gradients, burn-in, and incomplete mixing separate the implemented chain from its intended localized law.
    \item \textbf{Observation and Monte Carlo uncertainty:} replicated evaluation removes the conditional mean contribution of observation noise, while finite effective sample size still produces random error.
\end{itemize}

\subsection{Operational Consequence}
\label{app:h_m_operational_consequence}

The localization radius $h$ is therefore a \emph{scale index}, whereas the total chain length $M_{\mathrm{tot}}$ is a \emph{consistency control}. A single-$h$ curve can be informative as an operational, end-to-end diagnostic, but a multi-$h$ profile is needed to identify features that persist across scales and to expose tether- and slope-dependent effects. Increasing $M_{\mathrm{tot}}$ and examining multiple tail windows assess a different question: whether the chosen finite-scale target is being estimated stably. The following appendix reports the existing sensitivity sweeps under this interpretation.

\section{Finite-Step Operating Point and Scale Diagnostics}
\label{app:operational_diagnostics}

This appendix reports empirical sensitivity diagnostics for the finite-step and finite-scale effects discussed in Appendices~\ref{app:ula_bias} and \ref{app:h_m_sensitivity}. The sweeps provide an operating point for the reported experiments; they do not turn $h$ into a nuisance parameter or certify sampler convergence in a neural network.

\subsection{Stable Operating Point and Diagnostic Protocol}
\label{app:stable_operating_point}

To prevent confounded measurements, we anchor the diagnostics around an empirically established stable operating baseline. By freezing the system at this baseline, we sweep one structural parameter at a time across a wide numerical range, while holding the remaining parameters fixed. This protocol isolates the physical role of each component of the localized SGLD probe.

Unless otherwise specified, the controlled diagnostics are executed on the singular toy model $K_2(u,v)=u^2v^2$ with a hidden basin floor $L_0=100$ and independent Gaussian evaluation noise $\sigma=0.002$. The inverse-temperature scale is fixed explicitly at $t=10^4$, with total SGLD steps $M=100{,}000$ and burn-in ratio $\alpha=0.3$.

Table~\ref{tab:operational_defaults} summarizes the stable operating point used throughout the toy-model diagnostics, together with the corresponding deep-learning batch-allocation parameters. The toy diagnostics use $N=64$ as the default evaluation-group size in order to expose the statistical mechanics of the debiasing correction under controlled noise. In the neural-network experiments, we increase the evaluation group to $N=512$ using vectorized forward passes, as described in Appendix~\ref{app:dnn_vectorized_N}, to further suppress the residual uncertainty of the noise correction.

\begin{table}[ht]
\centering
\caption{Default probe configuration. Unless otherwise stated, these values are held fixed when sweeping a single structural parameter.}
\label{tab:operational_defaults}
\begin{tabular}{llc}
\toprule
Category & Quantity & Value \\
\midrule
Toy model & Localization and $N$ diagnostics & $K_2(u,v)=u^2v^2$ \\
Toy model & Finite-step calibration & $K_1(u)=u^2$ and $K_2(u,v)=u^2v^2$ \\
Toy model & Hidden basin floor $L_0$ & $100$ \\
Toy model & Evaluation noise standard deviation $\sigma$ & $0.002$ \\
Common & Inverse-temperature scale $t$ & $10{,}000$ \\
Common & Total SGLD steps $M$ & $100{,}000$ \\
Common & Burn-in ratio $\alpha$ & $0.3$ \\
Common & Independent MCMC trials & $5$ \\
Default toy setting & Localization radius $h$ & $2.0$ \\
Default toy setting & Base step size $\texttt{base\_lr}$ & $0.05$ \\
Default toy setting & Evaluation group size $N$ & $64$ \\
DNN setting & Relative localization multiplier $c_h$ & $1.0$ \\
DNN setting & Base step size $\texttt{base\_lr}$ & $0.05$ \\
DNN setting & Evaluation group size $N$ & $512$ \\
DNN setting & Evaluation mini-batch size $B_{\mathrm{eval}}$ & $64$ \\
DNN setting & Gradient mini-batch size $B_{\mathrm{grad}}$ & $1024$ \\
\bottomrule
\end{tabular}
\end{table}

These values define the operating point of the reported study. The subsequent sweeps perturb $h$, $\texttt{base\_lr}$, and $N$ around this baseline to show their observed sensitivities.

\subsection{Localization Sensitivity to $h$}
\label{app:h_sensitivity}
The localization radius $h$ should not be interpreted as a purely numerical hyperparameter. In the off-equilibrium setting, it defines the spatial scale of the local chart. If $h$ is too small, the resulting estimate is dominated by the tight tether and its interaction with the local slope. If $h$ is too large, the chain can move toward lower-loss regions outside the intended observation window. A checkpoint-local finite-scale diagnostic therefore requires declaring $h$ and inspecting whether the qualitative trajectory persists across a useful range of scales.

Having fixed the remaining configuration in Table~\ref{tab:operational_defaults}, we first vary the localization radius $h$. Table~\ref{tab:operational_h} is an empirical multi-scale profile of the end-to-end probe, including its finite-step sampler; it should not be read as a theorem that the continuous finite-temperature target is independent of $h$.

\begin{table}[ht]
\centering
\caption{Localization sensitivity to $h$ on the singular toy model $K_2(u,v)=u^2v^2$. SIVE values reproduce the original end-to-end implementation, including its nonnegative numerical clipping.}
\label{tab:operational_h}
\begin{tabular}{rcc}
\toprule
$h$ & Oracle Mean-based & SIVE \\
\midrule
$0.05$ & $1.7305 \pm 0.1120$ & $1.7330 \pm 0.1469$ \\
$0.10$ & $0.4942 \pm 0.0857$ & $0.5077 \pm 0.1172$ \\
$0.20$ & $0.3264 \pm 0.1011$ & $0.3376 \pm 0.1462$ \\
$0.50$ & $0.3763 \pm 0.0760$ & $0.4264 \pm 0.0969$ \\
$1.00$ & $0.3866 \pm 0.0576$ & $0.4427 \pm 0.0957$ \\
$2.00$ & $0.3970 \pm 0.0698$ & $0.4547 \pm 0.1040$ \\
$5.00$ & $0.3944 \pm 0.0665$ & $0.4551 \pm 0.1043$ \\
$10.00$ & $0.3938 \pm 0.0657$ & $0.4549 \pm 0.1041$ \\
$20.00$ & $0.3938 \pm 0.0658$ & $0.4550 \pm 0.1042$ \\
$50.00$ & $0.3938 \pm 0.0658$ & $0.4550 \pm 0.1042$ \\
\bottomrule
\end{tabular}
\end{table}

The sweep displays two useful operational regimes, but their interpretation is empirical:
\begin{itemize}[leftmargin=2em]
    \item \textbf{Tight-tether sensitivity ($h \le 0.5$):} The non-monotone values at very small $h$ show that localization, local slope, discrete step size, and incomplete exploration interact strongly in this regime. Proposition~\ref{prop:slope_plateau} rules out assigning these changes a universal $\mathcal O(h^{-2})$ law.
    \item \textbf{Empirical stability region ($h \ge 1.0$):} Within the resolution of these five-trial runs, the implemented SIVE values remain near $0.45$ from $h=1$ to $h=50$. This is a useful robustness observation for this sampler and toy problem. It is not proof of $h$-invariance or of a topological plateau; indeed, the exact continuous target in Appendix~\ref{app:free_energy_and_finite_temp} retains a slow logarithmic dependence on $h$.
\end{itemize}

\paragraph{Operational Translation for DNNs.}
The toy sweep uses an absolute Gaussian width in a two-parameter chart. The DNN probe instead uses a relative multiplier $c_h$ whose physical width depends on the checkpoint norm, as defined in Eq.~\eqref{eq:dnn_physical_h}. Thus the two numerical values called ``1'' are not a cross-model calibration. We declare $c_h=1$ as the main DNN convention and report a paired scale profile rather than selecting a favorable scale after inspection.

\begin{table}[ht]
\centering
\caption{Paired DNN localization audit. Values and contrasts are mean $\pm$ sample standard deviation across five training-trajectory-level quantities, after first averaging five probe trials per trajectory. $D_{\mathrm{Ckpt}}^{\mathrm{late}}$, $D_{\mathrm{Raw}}^{\mathrm{late}}$, and $D_{\mathrm{SIVE}}^{\mathrm{late}}$ are the corresponding Epoch-40-to-100 contrasts; a positive late contrast is called a rebound only when preceded by the reported early drop. ``SIVE Late $+$'' counts positive SIVE late contrasts across the five trajectories. The final column is the ratio of aggregate observation-noise correction to aggregate Raw Variance at Epoch 40. No unclipped SIVE estimate was negative.}
\label{tab:dnn_h_audit}
\resizebox{\textwidth}{!}{%
\begin{tabular}{rccccccccc}
\toprule
$c_h$ & SIVE$_1$ & SIVE$_{40}$ & SIVE$_{100}$ & $D_{\mathrm{SIVE}}^{\mathrm{drop}}$ & $D_{\mathrm{Ckpt}}^{\mathrm{late}}$ & $D_{\mathrm{Raw}}^{\mathrm{late}}$ & $D_{\mathrm{SIVE}}^{\mathrm{late}}$ & SIVE Late $+$ & $\hat C/\mathrm{Raw}_{40}$ \\
\midrule
$0.5$ & $1169.6\pm113.6$ & $64.5\pm2.6$ & $70.0\pm2.7$ & $1105.1\pm113.1$ & $47.5\pm10.5$ & $-345.9\pm11.8$ & $5.5\pm2.5$ & $5/5$ & $88.2\%$ \\
$1$   & $1385.0\pm79.1$  & $179.8\pm3.7$ & $228.1\pm10.0$ & $1205.2\pm80.1$ & $120.7\pm12.7$ & $-242.2\pm9.9$ & $48.3\pm8.5$ & $5/5$ & $77.5\%$ \\
$2$   & $1485.4\pm73.8$  & $463.1\pm15.6$ & $564.3\pm33.8$ & $1022.2\pm83.7$ & $261.1\pm15.6$ & $-8.3\pm39.0$ & $101.1\pm33.2$ & $5/5$ & $67.2\%$ \\
$4$   & $1246.1\pm64.8$  & $840.3\pm50.1$ & $893.8\pm53.1$ & $405.7\pm81.1$ & $417.0\pm22.6$ & $76.1\pm71.8$ & $53.4\pm73.8$ & $4/5$ & $62.5\%$ \\
\bottomrule
\end{tabular}%
}
\end{table}

The SIVE early-to-middle drop is positive in every one of the $20$ paired trajectory/scale comparisons. Its late rebound is positive for all trajectories at $c_h\in\{0.5,1,2\}$ and for four of five at $c_h=4$, while the checkpoint-anchored late contrast is positive in all $20$ comparisons; both magnitudes are plainly scale dependent. A positive checkpoint late contrast is not the same turning evidence: at the main scale, that response also rises over the early interval. The aggregate noise shares $(\hat C/\mathrm{Raw})$ at Epochs $(1,40,100)$ are $(57.0\%,88.2\%,65.0\%)$, $(55.1\%,77.5\%,59.2\%)$, $(53.2\%,67.2\%,59.8\%)$, and $(51.7\%,62.5\%,61.4\%)$ for increasing $c_h$. In particular, the Epoch-40 share decreases with $c_h$ in every training trajectory. This variation is expected because changing $c_h$ changes the localized path and therefore both terms in the decomposition; it also makes clear why the $77.5\%$ value is a property of the declared protocol rather than an empirical constant that can be transferred to another probe. Because the fixed SGLD step also makes the tether coefficient $\eta/h_k(c_h)^2$ vary with $c_h$, this is an end-to-end scale profile rather than an experiment that separately identifies thermodynamic scale dependence and discretization effects. The normalized RMS displacement remains close to one physical localization width throughout the audited points, but this diagnostic is not a step-size convergence certificate.

The full-data gradient diagnostic supplies a deliberately conservative counterpoint. The residualized late contrasts are $-45.4\pm45.3$, $-64.8\pm118.7$, $-818.3\pm875.9$, and $-8998.6\pm8404.3$ for $c_h\in\{0.5,1,2,4\}$, with positive signs in $0/5$, $2/5$, $1/5$, and $0/5$ trajectories, respectively. At $c_h=1$, the average Epoch-40-to-100 change decomposes exactly along the retained path as a residual change of $-64.8$, a linear-path-variance change of $-166.6$, and a change of $+279.7$ in the $2t^2\operatorname{Cov}_{\mathcal W}(R,\ell)$ term, summing to the observed SIVE change of $+48.3$. Because the covariance is substantial and the residualized statistic is highly scale sensitive, linear detrending is a diagnostic decomposition rather than a curvature estimator. The scale profile supports the non-monotonicity of the reported finite-scale statistic while explicitly limiting a stronger stationary-geometric interpretation.

\subsection{Finite-Step Calibration of \texorpdfstring{\texttt{base\_lr}}{base\_lr}}
\label{app:base_lr_sensitivity}

Because the Moment-Closure barrier (Appendix~\ref{app:ula_singular_barrier}) prevents calculating the exact Unadjusted Langevin Algorithm (ULA) target for singular geometries, we empirically ablate the dimensionless base step size ($\texttt{base\_lr} = \eta t$). We select a reported operating point using the cross-trial standard deviation of the privileged Oracle Mean estimator as one finite-run stability diagnostic.

\begin{table}[ht]
\centering
\caption{Finite-step sensitivity on the singular toy model $K_2(u,v)=u^2v^2$. Among the tested values, $\texttt{base\_lr}=0.05$ gives the smallest Oracle cross-trial standard deviation and is used as the reported operating point.}
\label{tab:operational_lr}
\begin{tabular}{rccc}
\toprule
$\texttt{base\_lr}$ & Regime Description & Oracle Mean-based & SIVE \\
\midrule
$0.0001$ & Severe Finite-Run Bias  & $3.3124 \pm 0.7967$ & $4.0592 \pm 2.9188$ \\
$0.0010$ & Very Slow Exploration      & $0.2959 \pm 0.2465$ & $0.2349 \pm 0.1881$ \\
$0.0050$ & Slow Exploration              & $0.3325 \pm 0.1669$ & $0.3747 \pm 0.3230$ \\
$0.0100$ & Improved Finite-Run Stability      & $0.3544 \pm 0.0791$ & $0.4046 \pm 0.2046$ \\
$\mathbf{0.0500}$ & \textbf{Selected Operating Point} & $\mathbf{0.3980 \pm 0.0424}$ & $\mathbf{0.4310 \pm 0.0961}$ \\
$0.1000$ & Moderate Discretization Shift   & $0.4538 \pm 0.0682$ & $0.4808 \pm 0.1063$ \\
$0.5000$ & Severe Discretization Instability & $0.5203 \pm 0.1248$ & $3.0707 \pm 4.8986$ \\
$1.0000$ & Numerical Divergence     & $0.5718 \pm 0.1379$ & $9.4489 \pm 17.1410$ \\
\bottomrule
\end{tabular}
\end{table}

As shown in Table~\ref{tab:operational_lr}, very small tested step sizes ($\le 0.005$) exhibit poor finite-run mixing, while large values ($\ge 0.5$) exhibit severe discrete-time instability (Appendix~\ref{app:ula_regular}). Among this finite grid, $\texttt{base\_lr}=0.05$ has the smallest Oracle cross-trial standard deviation. We therefore use it as an empirical operating point, not as a universal optimum or a convergence certificate.

\subsection{Evaluation-Group Sensitivity to $N$}
\label{app:N_sensitivity}

SIVE dynamically subtracts the empirical noise floor using the within-group sample variance $s_m^2$. By the Law of Total Variance, the raw variance of the grouped empirical losses absorbs a structural artifact equal to $\mathbb{E}[t^2 s_m^2 / N] = t^2 \sigma^2 / N$. For $t=10^4$ and $\sigma=0.002$, this uncorrected artifact scales to $400/N$, significantly dominating the $\mathcal{O}(1)$ geometric signal.

Our debiasing step subtracts this average noise profile $\hat{C}=t^2M_{\mathcal W}^{-1}\sum_{m\in\mathcal W}s_m^2/N$, where $M_{\mathcal W}$ is the number of retained post-burn-in steps. Under the idealized homoscedastic Gaussian and independent-step calculation, the standard deviation of the correction term alone is
\begin{equation}
    \mathrm{Std}[\hat{C}] = \left( \frac{400}{N} \right) \sqrt{\frac{2}{M_{\mathcal W}(N-1)}}.
\end{equation}
The reported sweep retains $M_{\mathcal W}=70{,}000$ states. This idealized formula describes only correction-term uncertainty, not the full uncertainty of SIVE.

Table~\ref{tab:operational_n} records an end-to-end sweep of SIVE and the privileged Oracle Mean statistic across varying $N$. Within each row and trial, both use the same MCMC path, but Oracle Mean targets $t\bar K_{\mathcal W}$ whereas unclipped SIVE targets $t^2S_{\mathcal W}^2(K)$; agreement between them is therefore not a calibration criterion. Across different $N$, evaluation and transition draws shared a pseudorandom stream, so changing $N$ also changed the sampled path. The table summarizes the resulting end-to-end sensitivity.

\begin{table}[ht]
\centering
\caption{End-to-end evaluation-group sensitivity on the singular toy model $K_2(u,v)=u^2v^2$. Entries are mean $\pm$ sample standard deviation across five trials. SIVE includes the positive-part numerical safeguard used in the reported runs.}
\label{tab:operational_n}
\begin{tabular}{rcc}
\toprule
$N$
& Oracle Mean
& SIVE \\
\midrule
$2$    & $0.4057 \pm 0.0483$ & $1.8260 \pm 1.1970$ \\
$4$    & $0.3651 \pm 0.0571$ & $0.5277 \pm 0.4193$ \\
$8$    & $0.4300 \pm 0.0772$ & $0.4323 \pm 0.3739$ \\
$16$   & $0.4300 \pm 0.0772$ & $0.3856 \pm 0.1033$ \\
$32$   & $0.3899 \pm 0.0474$ & $0.3834 \pm 0.0880$ \\
\textbf{64} & $\mathbf{0.3980 \pm 0.0424}$ & $\mathbf{0.4310 \pm 0.0961}$ \\
$128$  & $0.4286 \pm 0.0820$ & $0.4522 \pm 0.1262$ \\
$256$  & $0.4214 \pm 0.0699$ & $0.4613 \pm 0.1098$ \\
$512$  & $0.4960 \pm 0.0545$ & $0.5207 \pm 0.0930$ \\
$1024$ & $0.3872 \pm 0.0657$ & $0.3664 \pm 0.1059$ \\
$2048$ & $0.3941 \pm 0.0636$ & $0.3874 \pm 0.0804$ \\
\bottomrule
\end{tabular}
\end{table}

\section{Implementation and Scale Translation for Deep Neural Networks}
\label{app:dnn_implementation}

Operationalizing the thermodynamic-variance target on overparameterized neural networks introduces practical challenges: the true local minimum $L_k^\star$ is inaccessible mid-training, the SGLD path is finite and localized, and loss measurements are stochastic. To deploy the estimator under a declared finite-step protocol, we rely on an operating point informed by the controlled ablations, explicit decoupling of stochastic noise sources, and hardware-accelerated vectorization.

The trained classifier has architecture $784$--$128$--$128$--$10$, with ReLU activations after both hidden affine layers. We use the standard MNIST split (60,000 training and 10,000 test images), map pixels to $[0,1]$ with \texttt{ToTensor}, and optimize mean cross-entropy for 100 epochs using plain SGD with learning rate $0.01$, batch size 64, no momentum, and no weight decay. Train and test losses stored with the current checkpoints are full-dataset mean cross-entropies evaluated at the exact saved parameters.

\subsection{Thermodynamic Translation and Adaptive Localization}
\label{app:dnn_adaptive_localization}

The DNN experiments use a declared fixed-step protocol rather than calibrating the Unadjusted Langevin Algorithm (ULA) discretization bias. We use the empirical operating value $\texttt{base\_lr}=0.05$ identified in the controlled ablations (Appendix~\ref{app:base_lr_sensitivity}) as a fixed convention across the deep-learning experiments, not as a convergence guarantee. This choice is independent of the unknown additive floor, which cancels from the sampling law; the estimation-layer guarantee does not cover ULA discretization bias.

We distinguish the dimensionless relative multiplier $c_h$ from the physical Gaussian width used at checkpoint $k$:
\begin{equation}
    h_k(c_h) = c_h\frac{\|\theta_{t_k}\|}{\sqrt d},
    \label{eq:dnn_physical_h}
\end{equation}
where $d$ is the total number of parameters. This RMS convention makes $c_h$ a relative per-coordinate scale, while the actual width $h_k(c_h)$ remains checkpoint dependent. It does not confer general reparameterization invariance: the finite-scale probe is defined relative to the chosen Euclidean parameterization.

The five main trajectories in Section~\ref{sec:dnn} use $c_h=1$. The paired audit evaluates
\begin{equation}
    c_h\in\{0.5,1,2,4\}
\end{equation}
at the checkpoints $\{1,40,100\}$. These values change the estimand and are therefore reported as a profile in Table~\ref{tab:dnn_h_audit}, rather than pooled into a single estimate. The paired contrasts are the early-to-middle drop $D_{\mathrm{drop}}(c_h)=\widehat\Lambda_{1,c_h}-\widehat\Lambda_{40,c_h}$ and the late change $D_{\mathrm{late}}(c_h)=\widehat\Lambda_{100,c_h}-\widehat\Lambda_{40,c_h}$. For SIVE, the latter is a rebound when it is positive following the early drop. The primary correction comparison is Raw Variance versus SIVE under shared retained paths and evaluation groups. All three mean-energy reference choices are reported in the main trajectory; the multi-scale audit additionally tracks the checkpoint-reference contrast.

At each audited checkpoint, the implementation computes the gradient $g_k$ of the full empirical training loss in memory-bounded batches, avoiding the positive noise floor in the squared norm of a stochastic gradient estimate. For $\ell_m=g_k^\top(\theta_m-\theta_{t_k})$ and $R_m=\bar L_m-\ell_m$, the retained path obeys the exact algebraic decomposition
\begin{equation}
t^2S_{\mathcal W}^2(\bar L)
=t^2S_{\mathcal W}^2(R)
+t^2S_{\mathcal W}^2(\ell)
+2t^2\operatorname{Cov}_{\mathcal W}(R,\ell).
\label{eq:pathwise_slope_decomposition}
\end{equation}
Because subtracting $\ell_m$ does not change the within-state evaluation variance $s_m^2$, the same correction gives a residual SIVE statistic. Accordingly, the runner records the full-data $\|g_k\|^2$, $t^2h_k(c_h)^2\|g_k\|^2$, the three pathwise terms in Eq.~\eqref{eq:pathwise_slope_decomposition}, the uncorrected Raw Variance, and both original and residual SIVE values. This is a diagnostic decomposition of the realized path, not an exact nonlinear slope correction.

\subsection{Decoupling Gradient Dynamics and Observation Noise}
\label{app:dnn_noise_decoupling}

Neural network experiments rely on mini-batch subsampling, which injects stochasticity into both the SGLD parameter dynamics and the geometric loss measurements. To control these effects independently, we separate the gradient mini-batch size $B_{\mathrm{grad}}$ (used strictly for the Langevin drift) from the evaluation mini-batch size $B_{\mathrm{eval}}$ (used strictly to measure the local energy state). Gradient and evaluation indices are sampled uniformly with replacement from the 60,000 training examples. The experiment assigns independent pseudorandom streams to gradient-batch indices, Langevin Gaussian draws, and evaluation-batch indices. Stream seeds are deterministic functions of training trajectory, probe trial, and checkpoint, so the multi-$c_h$ audit reuses common random numbers without allowing $N$ or evaluation draws to alter the parameter dynamics.

To reduce stochastic-gradient perturbations of the intended drift, we use a relatively large gradient mini-batch, fixing $B_{\mathrm{grad}}=1024$ (approximately 1.7\% of the MNIST training data). This does not make the drift equal to the full empirical-loss gradient, so stochastic-gradient bias remains one of the sampler effects in Eq.~\eqref{eq:tail_error_decomposition}. In contrast, $B_{\mathrm{eval}}$ is controlled independently, as it governs the observation noise injected into the variance-based finite-scale measurement.

\subsection{Vectorized Scaling of the Evaluation Group $N$}
\label{app:dnn_vectorized_N}

At each retained state used by SIVE, the estimator requires $N$ independent evaluation mini-batches, each of size $B_{\mathrm{eval}}$. The reported runner records these groups at every SGLD step and discards burn-in only during estimation; this uniform implementation is not a mathematical requirement of the estimator. If $\sigma_{\mathrm{single}}^2(\theta_m)$ denotes the conditional variance of the per-sample loss at state $\theta_m$, the additive observation variance embedded within the grouped empirical mean scales according to:
\begin{equation}
    \operatorname{Var}(\bar L_m\mid\theta_m)
    = \frac{\sigma_{\mathrm{single}}^2(\theta_m)}{N B_{\mathrm{eval}}}.
\end{equation}
Consequently, the raw observation-noise floor is controlled by the total number of evaluated examples, $N B_{\mathrm{eval}}$. Under the idealized homoscedastic Gaussian calculation in Appendix~\ref{app:N_sensitivity}, the within-group sample variance $s_m^2$ has $N-1$ degrees of freedom, so increasing $N$ improves the precision of the empirical correction. For general non-Gaussian evaluation noise, that uncertainty also depends on higher moments and is not determined by $N-1$ alone. In the reported DNN protocol, $B_{\mathrm{eval}}$ is fixed and increasing $N$ trades additional evaluation cost for more precise within-state replication.

Because all $N$ evaluations share the same parameter state, the implementation uses \texttt{torch.func.functional\_call} to flatten indices of shape $(N,B_{\mathrm{eval}})$ into a batched evaluation dimension and then reshape the unreduced losses, avoiding $N$ sequential model invocations. This reduces launch overhead while retaining the underlying evaluation and memory costs; the toy diagnostics use $N=64$, whereas the neural experiments use $N=512$ with $B_{\mathrm{eval}}=64$.

The code stores all three mean-energy reference choices, Raw Variance, and unclipped SIVE from the same retained group means; the clipped positive part is a separate field. The CSV fields \texttt{Online} and \texttt{Retrospective} are the path-minimum and final-reference mean probes, respectively. The logger also records negative-estimate frequencies, empirical quantiles, $h_k(c_h)$, $\eta/h_k(c_h)^2$, the RMS parameter displacement, and displacement divided by $h_k(c_h)$. Train and test loss are evaluated at the exact saved parameters, including the random-initialization checkpoint, rather than writing an Epoch-0 placeholder. Each output directory contains the resolved configuration, seeds, Python/platform information and Git state, source-file hashes, raw per-trial CSV files, and compressed SGLD histories. The reported five-trajectory run contains no negative unclipped SIVE values and the RMS displacement remains within approximately $0.99$--$1.01$ times $h_k(c_h)$ at the audited checkpoints. These records define the reported protocol.
\clearpage
\section{Complete Five-Trajectory DNN Results}
\label{app:dnn_five_trajectory_full}
The table below reports every checkpoint from the five-trajectory experiment used in Figure~\ref{fig:mnist_dynamics}. Each entry is the mean $\pm$ sample standard deviation across the five independently trained networks after first averaging five probe trials within each network. The three mean-energy columns share Eq.~\eqref{eq:dnn_mean_references}; the final two columns are the direct Raw Variance--SIVE pair in Eq.~\eqref{eq:dnn_variance_pair}. The implementation fields \texttt{Online} and \texttt{Retrospective} correspond to ``Path min'' and ``Final ref.'' below.

For the aggregate curves, the Epoch-100 contrast from any intermediate checkpoint $k=30,\ldots,60$ ranges from $+41.8$ to $+53.0$ for SIVE and from $-400.8$ to $-85.3$ for Raw Variance, computed from the unrounded trajectory means. Thus the opposite late directions do not depend on choosing Epoch 40, consistent with a broad mid-training trough in SIVE; this aggregate check does not establish every corresponding within-trajectory comparison.

\begingroup
\scriptsize
\setlength{\tabcolsep}{2pt}
\begin{longtable}{@{}rccccccc@{}}
\caption{Complete five-trajectory DNN results at $c_h=1$.}\label{tab:dnn_five_trajectory_full}\\
\toprule
& & & \multicolumn{3}{c}{Mean energy: reference choice} & \multicolumn{2}{c}{Thermodynamic variance} \\
\cmidrule(lr){4-6}\cmidrule(lr){7-8}
Epoch & Train loss & Test loss & Ckpt. & Path min & Final ref. & Raw & SIVE \\
\midrule
\endfirsthead
\multicolumn{8}{c}{\tablename\ \thetable{} (continued)}\\
\toprule
& & & \multicolumn{3}{c}{Mean energy: reference choice} & \multicolumn{2}{c}{Thermodynamic variance} \\
\cmidrule(lr){4-6}\cmidrule(lr){7-8}
Epoch & Train loss & Test loss & Ckpt. & Path min & Final ref. & Raw & SIVE \\
\midrule
\endhead
\midrule
\multicolumn{8}{r}{Continued on next page}\\
\endfoot
\bottomrule
\endlastfoot
0 & $2.3027\pm0.0011$ & $2.3029\pm0.0013$ & $-15513.1\pm35.0$ & $319.0\pm10.9$ & $7068.5\pm40.6$ & $5709.3\pm89.8$ & $4650.1\pm92.2$ \\
1 & $0.7929\pm0.0393$ & $0.7741\pm0.0394$ & $-3161.1\pm314.7$ & $235.0\pm7.3$ & $4322.5\pm85.7$ & $3083.6\pm57.8$ & $1385.0\pm79.1$ \\
2 & $0.4382\pm0.0064$ & $0.4209\pm0.0062$ & $-647.6\pm55.5$ & $219.7\pm4.6$ & $3289.1\pm17.1$ & $2687.3\pm23.5$ & $630.1\pm36.1$ \\
3 & $0.3626\pm0.0021$ & $0.3499\pm0.0026$ & $-274.7\pm31.9$ & $212.7\pm4.1$ & $2905.5\pm20.7$ & $2552.9\pm25.0$ & $457.2\pm22.6$ \\
4 & $0.3255\pm0.0016$ & $0.3139\pm0.0030$ & $-144.9\pm22.6$ & $208.3\pm2.3$ & $2664.8\pm30.6$ & $2446.5\pm57.3$ & $396.9\pm17.5$ \\
5 & $0.3021\pm0.0023$ & $0.2934\pm0.0035$ & $-96.3\pm15.9$ & $205.7\pm1.7$ & $2479.4\pm38.5$ & $2351.5\pm56.1$ & $371.1\pm10.8$ \\
6 & $0.2811\pm0.0047$ & $0.2746\pm0.0061$ & $-44.8\pm11.7$ & $197.1\pm6.3$ & $2320.3\pm43.4$ & $2251.4\pm55.3$ & $349.5\pm6.8$ \\
7 & $0.2657\pm0.0045$ & $0.2603\pm0.0052$ & $-34.4\pm24.3$ & $193.8\pm3.3$ & $2176.6\pm46.1$ & $2140.2\pm55.4$ & $321.8\pm8.5$ \\
8 & $0.2494\pm0.0050$ & $0.2451\pm0.0050$ & $-3.3\pm8.6$ & $189.4\pm5.1$ & $2045.2\pm46.9$ & $2040.7\pm70.2$ & $303.2\pm24.7$ \\
9 & $0.2355\pm0.0046$ & $0.2327\pm0.0049$ & $14.8\pm9.5$ & $184.4\pm5.3$ & $1924.1\pm45.2$ & $1954.7\pm51.3$ & $295.5\pm10.0$ \\
10 & $0.2217\pm0.0050$ & $0.2208\pm0.0056$ & $41.0\pm12.4$ & $180.0\pm6.2$ & $1812.7\pm42.6$ & $1862.0\pm53.6$ & $281.3\pm20.7$ \\
11 & $0.2111\pm0.0061$ & $0.2103\pm0.0069$ & $43.9\pm33.9$ & $176.2\pm4.2$ & $1709.1\pm39.1$ & $1778.1\pm48.3$ & $271.5\pm17.2$ \\
12 & $0.1996\pm0.0044$ & $0.2000\pm0.0050$ & $64.0\pm18.8$ & $171.7\pm3.7$ & $1614.1\pm35.1$ & $1706.5\pm47.0$ & $266.5\pm18.6$ \\
13 & $0.1909\pm0.0039$ & $0.1919\pm0.0025$ & $63.0\pm16.4$ & $171.2\pm4.1$ & $1526.7\pm32.3$ & $1632.5\pm32.0$ & $256.3\pm9.6$ \\
14 & $0.1800\pm0.0020$ & $0.1826\pm0.0021$ & $91.4\pm9.8$ & $165.8\pm3.6$ & $1446.1\pm30.2$ & $1562.5\pm34.2$ & $243.3\pm10.7$ \\
15 & $0.1715\pm0.0029$ & $0.1738\pm0.0028$ & $102.6\pm7.0$ & $159.5\pm4.4$ & $1372.4\pm28.0$ & $1502.0\pm32.5$ & $237.0\pm9.6$ \\
16 & $0.1634\pm0.0038$ & $0.1672\pm0.0035$ & $116.0\pm13.9$ & $160.0\pm5.7$ & $1304.4\pm26.9$ & $1443.4\pm30.9$ & $227.6\pm11.7$ \\
17 & $0.1556\pm0.0029$ & $0.1604\pm0.0030$ & $130.7\pm4.9$ & $152.1\pm1.5$ & $1241.3\pm25.1$ & $1389.3\pm23.3$ & $219.7\pm1.2$ \\
18 & $0.1493\pm0.0033$ & $0.1549\pm0.0036$ & $135.1\pm9.9$ & $151.6\pm5.2$ & $1182.7\pm24.1$ & $1345.0\pm32.1$ & $218.8\pm10.7$ \\
19 & $0.1435\pm0.0019$ & $0.1492\pm0.0024$ & $138.8\pm13.3$ & $148.2\pm4.9$ & $1128.5\pm23.5$ & $1299.5\pm32.0$ & $213.3\pm10.1$ \\
20 & $0.1369\pm0.0026$ & $0.1441\pm0.0023$ & $154.8\pm2.2$ & $146.6\pm4.3$ & $1077.8\pm22.9$ & $1261.0\pm24.1$ & $212.3\pm6.2$ \\
21 & $0.1317\pm0.0026$ & $0.1401\pm0.0019$ & $159.1\pm7.9$ & $144.7\pm4.9$ & $1030.3\pm22.1$ & $1217.0\pm27.9$ & $204.0\pm5.2$ \\
22 & $0.1268\pm0.0027$ & $0.1359\pm0.0016$ & $163.1\pm12.6$ & $141.0\pm3.5$ & $985.8\pm21.1$ & $1180.9\pm26.2$ & $201.6\pm7.7$ \\
23 & $0.1216\pm0.0025$ & $0.1316\pm0.0020$ & $173.5\pm8.9$ & $139.7\pm2.7$ & $944.2\pm20.7$ & $1148.7\pm22.9$ & $200.7\pm6.9$ \\
24 & $0.1179\pm0.0025$ & $0.1288\pm0.0030$ & $171.7\pm5.4$ & $141.6\pm4.6$ & $904.9\pm19.9$ & $1114.1\pm24.2$ & $196.1\pm7.3$ \\
25 & $0.1129\pm0.0024$ & $0.1243\pm0.0015$ & $184.3\pm9.6$ & $134.9\pm4.3$ & $868.1\pm19.7$ & $1085.9\pm27.3$ & $195.4\pm6.1$ \\
26 & $0.1086\pm0.0021$ & $0.1215\pm0.0024$ & $192.8\pm6.7$ & $132.6\pm3.6$ & $833.6\pm19.0$ & $1057.8\pm19.7$ & $193.5\pm4.2$ \\
27 & $0.1047\pm0.0024$ & $0.1180\pm0.0016$ & $199.3\pm5.7$ & $133.1\pm2.1$ & $800.7\pm18.8$ & $1028.2\pm23.0$ & $187.8\pm3.0$ \\
28 & $0.1008\pm0.0024$ & $0.1149\pm0.0026$ & $206.9\pm8.2$ & $132.8\pm4.5$ & $769.8\pm18.2$ & $1006.8\pm25.9$ & $190.4\pm5.9$ \\
29 & $0.0977\pm0.0021$ & $0.1132\pm0.0021$ & $209.1\pm9.0$ & $129.0\pm4.3$ & $740.7\pm17.7$ & $984.1\pm19.6$ & $188.4\pm6.3$ \\
30 & $0.0950\pm0.0023$ & $0.1116\pm0.0024$ & $209.4\pm8.3$ & $127.1\pm3.9$ & $713.3\pm17.8$ & $959.6\pm19.8$ & $183.4\pm4.3$ \\
31 & $0.0920\pm0.0020$ & $0.1089\pm0.0021$ & $212.5\pm7.3$ & $128.6\pm5.5$ & $687.3\pm17.1$ & $942.3\pm18.5$ & $186.0\pm3.0$ \\
32 & $0.0889\pm0.0020$ & $0.1072\pm0.0017$ & $219.1\pm4.9$ & $124.4\pm3.2$ & $662.2\pm16.5$ & $917.8\pm19.5$ & $180.9\pm2.6$ \\
33 & $0.0861\pm0.0015$ & $0.1054\pm0.0012$ & $223.3\pm6.4$ & $123.8\pm6.5$ & $638.6\pm15.7$ & $901.7\pm21.7$ & $182.0\pm6.7$ \\
34 & $0.0834\pm0.0016$ & $0.1033\pm0.0014$ & $227.9\pm12.4$ & $120.1\pm4.7$ & $616.5\pm15.3$ & $886.9\pm19.8$ & $183.8\pm5.4$ \\
35 & $0.0805\pm0.0018$ & $0.1016\pm0.0018$ & $235.6\pm5.3$ & $120.4\pm4.5$ & $595.4\pm14.8$ & $864.6\pm16.7$ & $176.3\pm4.0$ \\
36 & $0.0789\pm0.0021$ & $0.1010\pm0.0027$ & $231.3\pm10.6$ & $117.8\pm3.7$ & $574.7\pm14.6$ & $850.9\pm19.5$ & $178.9\pm4.3$ \\
37 & $0.0760\pm0.0023$ & $0.0984\pm0.0021$ & $241.6\pm8.7$ & $118.7\pm3.1$ & $555.9\pm14.4$ & $839.2\pm19.2$ & $180.0\pm5.9$ \\
38 & $0.0738\pm0.0017$ & $0.0969\pm0.0019$ & $244.1\pm7.6$ & $115.7\pm1.7$ & $536.8\pm13.6$ & $824.3\pm17.9$ & $178.7\pm7.8$ \\
39 & $0.0714\pm0.0014$ & $0.0957\pm0.0024$ & $251.4\pm8.7$ & $115.5\pm3.0$ & $519.4\pm13.5$ & $810.2\pm18.8$ & $177.2\pm3.8$ \\
40 & $0.0696\pm0.0022$ & $0.0943\pm0.0020$ & $252.7\pm8.5$ & $116.8\pm3.6$ & $502.8\pm12.8$ & $800.9\pm15.9$ & $179.8\pm3.7$ \\
41 & $0.0676\pm0.0016$ & $0.0932\pm0.0011$ & $255.9\pm6.1$ & $113.2\pm3.6$ & $486.7\pm13.0$ & $787.4\pm14.7$ & $178.2\pm3.3$ \\
42 & $0.0665\pm0.0020$ & $0.0932\pm0.0025$ & $252.0\pm7.9$ & $113.6\pm3.9$ & $471.0\pm12.5$ & $773.6\pm17.6$ & $176.1\pm5.7$ \\
43 & $0.0645\pm0.0013$ & $0.0918\pm0.0024$ & $256.9\pm14.1$ & $111.6\pm3.3$ & $456.7\pm12.1$ & $764.3\pm15.7$ & $177.2\pm3.0$ \\
44 & $0.0620\pm0.0017$ & $0.0897\pm0.0010$ & $268.1\pm8.0$ & $109.9\pm2.4$ & $442.3\pm11.7$ & $757.7\pm15.3$ & $180.3\pm2.2$ \\
45 & $0.0607\pm0.0008$ & $0.0897\pm0.0021$ & $267.5\pm9.8$ & $111.2\pm2.8$ & $428.7\pm11.5$ & $743.4\pm15.6$ & $176.5\pm2.1$ \\
46 & $0.0593\pm0.0018$ & $0.0886\pm0.0013$ & $268.9\pm12.8$ & $111.6\pm2.4$ & $415.7\pm10.9$ & $739.1\pm13.7$ & $181.1\pm3.4$ \\
47 & $0.0576\pm0.0015$ & $0.0879\pm0.0023$ & $272.8\pm5.8$ & $109.2\pm2.7$ & $403.3\pm10.7$ & $726.7\pm12.1$ & $177.9\pm4.3$ \\
48 & $0.0560\pm0.0022$ & $0.0868\pm0.0017$ & $277.0\pm13.5$ & $107.3\pm3.2$ & $391.4\pm10.6$ & $721.4\pm16.1$ & $181.6\pm3.7$ \\
49 & $0.0544\pm0.0013$ & $0.0853\pm0.0015$ & $281.7\pm2.2$ & $110.5\pm3.3$ & $379.7\pm10.2$ & $707.0\pm13.6$ & $175.1\pm2.4$ \\
50 & $0.0528\pm0.0010$ & $0.0847\pm0.0016$ & $286.1\pm6.1$ & $106.3\pm3.8$ & $368.3\pm10.2$ & $703.9\pm14.6$ & $180.0\pm4.3$ \\
51 & $0.0515\pm0.0016$ & $0.0843\pm0.0018$ & $288.1\pm7.8$ & $106.1\pm2.5$ & $357.6\pm10.0$ & $695.8\pm17.3$ & $179.5\pm4.9$ \\
52 & $0.0500\pm0.0015$ & $0.0831\pm0.0010$ & $292.6\pm4.2$ & $108.3\pm2.8$ & $347.4\pm9.6$ & $688.0\pm12.5$ & $179.5\pm3.2$ \\
53 & $0.0490\pm0.0012$ & $0.0833\pm0.0016$ & $293.3\pm6.7$ & $107.3\pm3.4$ & $337.4\pm9.2$ & $684.0\pm14.0$ & $182.2\pm3.2$ \\
54 & $0.0479\pm0.0011$ & $0.0831\pm0.0007$ & $294.4\pm9.3$ & $105.5\pm4.4$ & $327.6\pm9.3$ & $675.9\pm14.0$ & $181.8\pm5.2$ \\
55 & $0.0477\pm0.0016$ & $0.0828\pm0.0019$ & $287.4\pm3.3$ & $104.1\pm1.9$ & $318.5\pm8.9$ & $670.1\pm12.5$ & $182.5\pm3.6$ \\
56 & $0.0462\pm0.0010$ & $0.0824\pm0.0016$ & $293.2\pm8.4$ & $105.8\pm4.9$ & $309.2\pm9.0$ & $663.4\pm13.8$ & $182.4\pm4.7$ \\
57 & $0.0440\pm0.0011$ & $0.0803\pm0.0008$ & $305.8\pm5.6$ & $103.0\pm3.7$ & $300.6\pm8.6$ & $660.7\pm12.3$ & $185.4\pm3.7$ \\
58 & $0.0431\pm0.0009$ & $0.0802\pm0.0017$ & $306.8\pm5.3$ & $101.9\pm1.7$ & $292.1\pm8.6$ & $650.4\pm13.8$ & $181.6\pm4.2$ \\
59 & $0.0422\pm0.0014$ & $0.0805\pm0.0021$ & $307.7\pm4.9$ & $99.9\pm4.2$ & $283.9\pm8.6$ & $646.9\pm10.1$ & $183.6\pm3.9$ \\
60 & $0.0414\pm0.0022$ & $0.0793\pm0.0017$ & $307.7\pm10.4$ & $102.2\pm3.2$ & $276.2\pm8.1$ & $644.1\pm12.8$ & $186.4\pm4.4$ \\
61 & $0.0397\pm0.0011$ & $0.0792\pm0.0009$ & $317.7\pm3.5$ & $102.8\pm3.6$ & $268.7\pm8.1$ & $636.6\pm11.2$ & $184.0\pm6.1$ \\
62 & $0.0392\pm0.0010$ & $0.0795\pm0.0016$ & $314.6\pm8.5$ & $101.7\pm1.8$ & $261.3\pm8.0$ & $635.1\pm11.3$ & $187.9\pm3.6$ \\
63 & $0.0382\pm0.0013$ & $0.0784\pm0.0013$ & $317.1\pm10.1$ & $101.7\pm5.0$ & $253.9\pm7.8$ & $630.7\pm16.1$ & $189.3\pm6.3$ \\
64 & $0.0379\pm0.0016$ & $0.0782\pm0.0015$ & $313.7\pm16.4$ & $103.3\pm4.2$ & $246.8\pm7.8$ & $625.6\pm14.4$ & $188.9\pm3.7$ \\
65 & $0.0367\pm0.0016$ & $0.0784\pm0.0020$ & $319.4\pm6.7$ & $100.6\pm3.1$ & $240.5\pm7.8$ & $622.7\pm12.3$ & $190.6\pm5.9$ \\
66 & $0.0360\pm0.0012$ & $0.0778\pm0.0007$ & $318.8\pm5.8$ & $100.5\pm3.0$ & $233.5\pm7.5$ & $616.5\pm11.8$ & $189.0\pm3.0$ \\
67 & $0.0352\pm0.0014$ & $0.0783\pm0.0011$ & $320.9\pm8.0$ & $98.9\pm1.2$ & $227.4\pm7.4$ & $612.2\pm14.1$ & $189.0\pm6.6$ \\
68 & $0.0339\pm0.0010$ & $0.0771\pm0.0007$ & $327.5\pm4.4$ & $99.8\pm1.7$ & $221.3\pm7.4$ & $609.2\pm12.8$ & $190.3\pm3.7$ \\
69 & $0.0338\pm0.0012$ & $0.0772\pm0.0008$ & $322.6\pm2.4$ & $99.7\pm4.6$ & $215.1\pm7.0$ & $607.3\pm12.5$ & $193.6\pm3.3$ \\
70 & $0.0328\pm0.0016$ & $0.0773\pm0.0024$ & $327.2\pm15.5$ & $98.2\pm2.5$ & $209.1\pm6.8$ & $601.8\pm15.0$ & $192.0\pm5.3$ \\
71 & $0.0318\pm0.0010$ & $0.0766\pm0.0011$ & $331.7\pm10.0$ & $99.5\pm3.2$ & $203.7\pm6.8$ & $602.3\pm11.3$ & $196.3\pm6.3$ \\
72 & $0.0307\pm0.0008$ & $0.0756\pm0.0014$ & $336.9\pm7.5$ & $99.4\pm3.1$ & $198.2\pm6.5$ & $598.7\pm12.8$ & $196.3\pm3.9$ \\
73 & $0.0300\pm0.0008$ & $0.0761\pm0.0013$ & $338.8\pm8.0$ & $96.8\pm3.0$ & $192.8\pm6.5$ & $594.3\pm13.6$ & $196.1\pm7.5$ \\
74 & $0.0296\pm0.0009$ & $0.0754\pm0.0013$ & $336.9\pm5.2$ & $98.6\pm3.0$ & $187.5\pm6.4$ & $592.2\pm10.6$ & $197.0\pm4.4$ \\
75 & $0.0290\pm0.0007$ & $0.0758\pm0.0006$ & $338.6\pm9.3$ & $97.3\pm3.9$ & $182.5\pm6.2$ & $594.5\pm12.8$ & $202.5\pm7.3$ \\
76 & $0.0279\pm0.0010$ & $0.0751\pm0.0015$ & $344.6\pm6.3$ & $98.7\pm2.3$ & $177.6\pm6.1$ & $587.7\pm11.1$ & $199.9\pm6.2$ \\
77 & $0.0274\pm0.0010$ & $0.0760\pm0.0014$ & $344.1\pm7.7$ & $97.0\pm3.2$ & $172.7\pm5.9$ & $585.6\pm13.3$ & $201.1\pm5.2$ \\
78 & $0.0267\pm0.0010$ & $0.0750\pm0.0016$ & $346.6\pm8.9$ & $95.1\pm2.2$ & $168.1\pm5.9$ & $581.5\pm11.7$ & $200.0\pm3.0$ \\
79 & $0.0263\pm0.0011$ & $0.0760\pm0.0022$ & $345.8\pm7.4$ & $97.2\pm2.8$ & $163.6\pm6.0$ & $583.0\pm11.2$ & $204.2\pm6.6$ \\
80 & $0.0261\pm0.0012$ & $0.0750\pm0.0011$ & $344.0\pm6.2$ & $97.0\pm2.5$ & $159.2\pm5.6$ & $577.8\pm13.0$ & $202.9\pm5.0$ \\
81 & $0.0260\pm0.0016$ & $0.0757\pm0.0024$ & $341.3\pm18.2$ & $96.9\pm1.7$ & $155.1\pm5.9$ & $578.2\pm10.3$ & $205.7\pm3.0$ \\
82 & $0.0245\pm0.0015$ & $0.0746\pm0.0018$ & $351.4\pm6.7$ & $97.5\pm3.5$ & $150.9\pm5.7$ & $571.8\pm12.2$ & $202.3\pm4.1$ \\
83 & $0.0239\pm0.0007$ & $0.0748\pm0.0012$ & $352.8\pm6.5$ & $93.2\pm2.2$ & $146.6\pm5.5$ & $578.2\pm11.7$ & $211.7\pm3.9$ \\
84 & $0.0234\pm0.0010$ & $0.0748\pm0.0011$ & $354.0\pm5.7$ & $95.2\pm1.5$ & $142.6\pm5.2$ & $572.7\pm12.8$ & $208.7\pm6.8$ \\
85 & $0.0228\pm0.0010$ & $0.0742\pm0.0014$ & $356.9\pm9.5$ & $94.0\pm2.7$ & $138.8\pm5.0$ & $569.7\pm12.9$ & $208.3\pm6.5$ \\
86 & $0.0222\pm0.0007$ & $0.0745\pm0.0012$ & $358.1\pm8.0$ & $96.2\pm3.0$ & $134.7\pm4.9$ & $569.3\pm13.3$ & $210.4\pm6.0$ \\
87 & $0.0218\pm0.0007$ & $0.0749\pm0.0010$ & $358.7\pm6.0$ & $94.3\pm2.6$ & $131.1\pm5.3$ & $567.0\pm12.3$ & $210.6\pm5.4$ \\
88 & $0.0213\pm0.0006$ & $0.0739\pm0.0014$ & $361.0\pm6.2$ & $96.9\pm1.9$ & $128.0\pm4.9$ & $566.5\pm12.2$ & $212.3\pm4.2$ \\
89 & $0.0208\pm0.0007$ & $0.0748\pm0.0010$ & $361.9\pm6.8$ & $92.8\pm1.6$ & $124.1\pm5.2$ & $568.0\pm8.8$ & $216.5\pm2.3$ \\
90 & $0.0204\pm0.0007$ & $0.0745\pm0.0017$ & $362.0\pm7.9$ & $94.3\pm0.9$ & $120.5\pm4.7$ & $564.2\pm14.0$ & $214.7\pm6.3$ \\
91 & $0.0201\pm0.0010$ & $0.0747\pm0.0017$ & $362.3\pm11.9$ & $93.8\pm3.2$ & $117.8\pm4.8$ & $564.9\pm14.0$ & $217.4\pm5.9$ \\
92 & $0.0196\pm0.0010$ & $0.0745\pm0.0015$ & $364.2\pm4.3$ & $94.0\pm1.7$ & $114.5\pm4.7$ & $561.1\pm14.4$ & $216.0\pm7.3$ \\
93 & $0.0192\pm0.0010$ & $0.0746\pm0.0014$ & $364.7\pm8.6$ & $94.3\pm5.4$ & $111.2\pm4.5$ & $564.9\pm14.9$ & $221.3\pm7.3$ \\
94 & $0.0185\pm0.0008$ & $0.0738\pm0.0012$ & $368.0\pm10.7$ & $92.7\pm2.2$ & $107.8\pm4.7$ & $557.7\pm10.3$ & $217.0\pm4.6$ \\
95 & $0.0183\pm0.0004$ & $0.0744\pm0.0009$ & $368.0\pm10.0$ & $95.2\pm2.1$ & $105.1\pm4.4$ & $557.6\pm8.7$ & $218.2\pm2.1$ \\
96 & $0.0178\pm0.0007$ & $0.0739\pm0.0011$ & $370.4\pm6.1$ & $94.8\pm2.8$ & $102.4\pm4.5$ & $557.8\pm16.4$ & $220.3\pm9.7$ \\
97 & $0.0174\pm0.0005$ & $0.0745\pm0.0017$ & $370.3\pm9.5$ & $95.2\pm2.2$ & $99.0\pm4.4$ & $558.7\pm10.5$ & $222.8\pm5.3$ \\
98 & $0.0179\pm0.0017$ & $0.0760\pm0.0032$ & $363.1\pm9.6$ & $92.7\pm2.5$ & $96.4\pm4.4$ & $561.3\pm14.6$ & $226.9\pm8.2$ \\
99 & $0.0169\pm0.0007$ & $0.0756\pm0.0014$ & $370.1\pm6.2$ & $93.9\pm2.3$ & $93.7\pm4.1$ & $556.9\pm16.0$ & $224.4\pm9.8$ \\
100 & $0.0163\pm0.0007$ & $0.0747\pm0.0012$ & $373.4\pm6.3$ & $90.9\pm4.1$ & $90.9\pm4.1$ & $558.8\pm15.5$ & $228.1\pm10.0$ \\
\end{longtable}
\endgroup

\end{document}